\documentclass{article} 
\usepackage{iclr_conference,times}

\usepackage{amsmath,amsfonts,bm}

\def\eqref#1{equation~\ref{#1}}

\def\1{\bm{1}}

\DeclareMathAlphabet{\mathsfit}{\encodingdefault}{\sfdefault}{m}{sl}
\SetMathAlphabet{\mathsfit}{bold}{\encodingdefault}{\sfdefault}{bx}{n}

\usepackage{hyperref}
\usepackage{url}
\usepackage{amsthm}
\usepackage{amssymb}
\usepackage{algorithm}
\usepackage{algpseudocode}
\usepackage{graphicx}
\usepackage{enumitem}
\usepackage{subcaption}
\usepackage{booktabs}
\usepackage{tabularx}
\usepackage{xurl}
\usepackage{placeins}
\usepackage{float}

\usepackage{threeparttable}
\usepackage{placeins}

\usepackage{geometry}
\usepackage{array}
\usepackage{longtable}
\usepackage{booktabs}

\newcolumntype{L}[1]{>{\raggedright\arraybackslash}p{#1}}

\newtheorem{proposition}{Proposition}
\newtheorem{definition}{Definition}

\usepackage{tikz}
\tikzset{
  every node/.style={font=\scriptsize},
  box/.style={
    draw, rounded corners=2pt, line width=0.30pt,
    align=center, inner xsep=3.5pt, inner ysep=2.6pt,
    minimum height=5.4mm
  },
  group/.style={draw, rounded corners=3pt, line width=0.30pt, inner sep=4.5pt},
  gtitle/.style={font=\bfseries\scriptsize, inner sep=1pt},
  arrow/.style={
    -{Latex[length=1.8mm]}, line width=0.35pt,
    rounded corners=2pt,
    shorten >=2pt, shorten <=2pt
  }
}

\title{Integrable Elasticity \\ via Neural Demand Potentials}

\author{
Carlos Heredia\thanks{e-mail address: carlosherediapimienta@gmail.com}\hspace{0.5cm}\&\hspace{0.5cm}Daniel Roncel\thanks{e-mail address: danielronceldiaz@gmail.com} \\
\textit{IAMM Research, Department of Applied Artificial Intelligence} \\
\textit{DAMM, Carrer del Rossell\'o 515, 08025 Barcelona, Catalonia, Spain}
}

\iclrfinalcopy
\begin{document}

\maketitle

\begin{abstract}
We propose the Integrable Context-Dependent Demand Network (ICDN), a demand-first neural model for multiproduct retail demand. ICDN learns log-demand as a smooth, context-conditioned function of log-prices, allowing elasticities to be derived exactly from the learned demand surface. Across three retail datasets, the model achieves competitive out-of-sample prediction while producing analytically tractable, and economically regularized own- and cross-price responses.
\end{abstract}


\section{Introduction}\label{sec:intro}

A central input to pricing decisions is a simple but difficult question: how does demand respond when prices change? In practice, this question is answered through elasticities---local log-sensitivities of demand with respect to price (\cite{Asano_2012})---which are the natural primitives for scenario simulation, revenue optimization, and promotion planning. However, learning stable and reproducible elasticities at scale remains challenging. Retail demand is heterogeneous across stores and products, nonstationary over time, and shaped by promotions and assortment effects. At the same time, prices are often correlated across stock-keeping units (SKUs) and governed by business rules, making it easy for flexible models to fit demand while producing unstable or economically implausible elasticities.

From a mathematical perspective, elasticities are derivatives of the demand map (\cite{Anthony_Biggs_2024}). As a result, small-scale irregularities in an estimated demand surface can be amplified into large fluctuations in elasticities. The challenge is more pronounced in the multiproduct setting, where the elasticity matrix couples all SKUs---or, in practice, at least the most relevant competitors. Additionally, cross-price effects are especially difficult to identify from observational price variation because they are often weaker than own-price responses and blurred by correlated pricing and promotional activity. 

In applied work, elasticities are therefore typically computed from an estimated demand model, ranging from simple log-log OLS benchmarks (\cite{Baddeley_Barrowclough_2009}) to modern deep-learning predictors (\cite{AcunaAgost2021PriceElasticityDeepLearning}). From this perspective, it is natural to seek a single structured demand model that is differentiable and analytically tractable, so that elasticities can be obtained exactly as derivatives of the same underlying surface (\cite{Berry1995, BanksBlundellLewbel1997}).

\subsection{Research gap and overall research framework}
Existing approaches to multiproduct elasticity estimation face a persistent trade-off between interpretability, flexibility, derivative coherence, and scalability. Classical log-log demand models provide directly interpretable elasticities, but generally impose constant or highly restricted price responses. Generic machine-learning demand predictors can capture nonlinear and context-dependent patterns, but they are not typically designed to produce transparent and analytically tractable elasticities. These limitations, together with the weak identification of cross-price effects and the rapid growth in potential product interactions, motivate a structured demand-first framework that combines nonlinear contextual flexibility with coherent, tractable, and regularized price responses. 

To address these gaps, we introduce the Integrable Context-Dependent Demand Network (ICDN), a structured neural demand model that represents log-demand as a smooth function of log-prices and contextual covariates, so that own- and cross-price elasticities arise as exact derivatives of the same underlying surface.

This demand-first construction is not unique to ICDN: any sufficiently smooth model of log-demand induces an integrable elasticity field when elasticities are defined as its Jacobian. The contribution of ICDN is to implement this principle in a form tailored to high-dimensional retail pricing, combining an explicit price basis, analytic spline derivatives, economic regularization, and sparse directed cross-price interactions.

Structurally, ICDN decomposes price effects into interpretable linear components and flexible spline-based nonlinear refinements. Own-price responses are represented by a locally linear log-log term \(\beta_{ii}(x)u_i\) together with a nonlinear spline component \(w_{ii}(x)^\top B_i(u_i)\). Directed cross-price responses are modeled through a linear component \(\beta_{ij}(x)u_j\), a nonlinear univariate refinement \(w_{ij}(x)^\top B_j(u_j)\), and a spline-spline interaction \(B_i(u_i)^\top U^{(ij)}(x)B_j(u_j)\) that captures nonlinear relative price-positioning effects between products \(i\) and \(j\).

The coefficients attached to these price-response structures are not fixed globally: they vary with the observed retail context. A shared encoder maps the contextual information associated with each SKU into a product-specific latent representation \(h_i\). Own-price parameters are derived from \(h_i\), while cross-price parameters are constructed from ordered pairs \((h_i,h_j)\) and modulated by attention weights over selected product pairs. This ordered-pair construction does not impose symmetry: \(E_{ij}\) and \(E_{ji}\) are learned independently, allowing directional substitution or complementarity patterns across products. Because the spline bases admit closed-form derivatives, the resulting elasticities remain analytically computable without dense Jacobian or Hessian evaluations. Table~\ref{tab:research-framework} summarizes the relationship between the identified research gaps and the corresponding components of ICDN, and clarifies how these components operate jointly within the framework.
\begin{table}[t!]
\centering
\scriptsize
\setlength{\tabcolsep}{3pt}
\renewcommand{\arraystretch}{1.18}
\caption{Overall research framework linking the identified research gaps to
the corresponding ICDN components, their methodological roles, and their
contribution within the proposed model.}
\label{tab:research-framework}

\begin{tabularx}{\textwidth}{
@{}
p{0.205\textwidth}
p{0.195\textwidth}
p{0.285\textwidth}
X
@{}}
\toprule

\textbf{Research gap}
&
\textbf{Methodological response}
&
\textbf{Role and connection within ICDN}
&
\textbf{Contribution to addressing the gap}
\\

\midrule

Elasticities estimated as independent local objects may not correspond to a single coherent demand function.
&
\textit{Demand-first} formulation
&
ICDN first defines a differentiable context-conditioned log-demand surface \(g_{\theta}(u,x)\). Own- and cross-price elasticities are subsequently obtained from the same surface as \(E_{ij}(u,x)=\partial g_{\theta,i}(u,x)/\partial u_j\).
&
Provides derivative coherence and row-wise integrability by construction. This is a general demand-first principle that ICDN operationalizes in a structured multiproduct retail architecture.
\\

\addlinespace[2pt]

Classical log-log models impose constant or highly restricted price effects and cannot represent state-dependent price responses.
&
Context-dependent coefficients
&
A shared encoder produces one latent representation per SKU. These representations generate product- and context-specific baseline, own-price, and cross-price coefficients.
&
Allows price sensitivities to vary across products, stores, seasons, promotional regimes, and other observed market conditions.
\\

\addlinespace[2pt]

Simple global functional forms provide insufficient flexibility for nonlinear price-response patterns.
&
Product-specific spline price basis
&
Fixed cubic spline bases represent localized nonlinear variation over each SKU's empirical log-price support. Their coefficients are generated by the context branch.
&
Captures nonlinear and locally varying own- and cross-price responses without sacrificing an explicit price representation.
\\

\addlinespace[2pt]

Elasticities derived from generic neural predictors may be computationally opaque and require dense automatic-differentiation operations.
&
Analytic spline derivatives
&
The price branch evaluates \(B_i(u_i)\), \(B_i'(u_i)\), and \(B_i''(u_i)\). These quantities are used jointly in demand prediction, elasticity calculation, and surface regularization.
&
Makes the contribution of each price component explicit and permits closed-form calculation of the required elasticities and curvature terms.
\\

\addlinespace[2pt]

The number of possible cross-product relationships increases rapidly with the number of products, while many interactions are weak or economically irrelevant.
&
Sparse directed attention
&
Attention scores select a limited set of ordered product pairs for each focal SKU. Observation-specific weights then modulate the selected cross-price interactions.
&
Restricts the evaluated cross-price structure to relevant directed relationships and allows the effect of product \(j\)'s price on product \(i\)'s demand to differ from the reverse effect.
\\

\addlinespace[2pt]

Flexible demand models may fit observed demand through irregular surfaces, producing unstable or economically extreme derivatives.
&
Curvature and elasticity-band regularization
&
The curvature penalty acts on the geometry of the log-demand surface, whereas the elasticity-band penalty acts directly on the model-implied own- and cross-price responses.
&
Encourages smoother demand surfaces and discourages implausibly large local elasticities while retaining soft, data-dependent constraints.
\\

\addlinespace[2pt]

Training the complete nonlinear architecture from a random initialization may generate unstable initial price responses.
&
Two-phase training protocol
&
A context-dependent log-linear warm start is first estimated on smoothed targets. Nonlinear spline and interaction components are subsequently unfrozen and trained on the original demand observations.
&
Provides an economically meaningful initialization before introducing the full nonlinear flexibility of ICDN.
\\

\bottomrule
\end{tabularx}
\end{table}
In our primary empirical application, we evaluate ICDN on the Dominick's Finer Foods scanner data for the beer category, a weekly retail scanner panel covering store-specific sales, prices, and promotional activity (\cite{kilts_dominicks_manual_2018,kilts_dominicks_dataset}). To make demand responses comparable across products, we express quantity in liters sold and prices in price per liter, and work in log-log coordinates. In this setting, own-price elasticities are expected to be predominantly negative, a pattern that is both standard for this class of products and confirmed in our cleaned Dominick's sample (\cite{ToroGonzalezMcCluskeyMittelhammer2014}). We additionally assess robustness under a matched evaluation protocol on Walmart M5 and Dunnhumby, together with a larger Dominick product panel.

Because the empirical application uses observational scanner data, the elasticities reported below should not be interpreted as fully identified causal effects of price interventions. They are local, context-conditioned derivatives of the fitted demand surface, evaluated under the current pricing regime. Accordingly, our empirical analysis focuses on predictive generalization, derivative coherence, and elasticity stability, rather than on causal identification. For causal approaches to pricing, see, for instance, \cite{causal-pricing} and \cite{schultz2024causalforecastingpricing}.

Training balances predictive accuracy with economic regularity. We fit log-demand with a robust objective to reduce sensitivity to outliers and heteroskedastic noise, while regularizing the geometry of the learned demand surface to avoid artificial oscillations in the implied elasticities. In addition, we impose a soft penalty on positive own-price elasticities, reflecting the dominant empirical pattern of downward-sloping demand in this application, while still allowing local deviations when sufficiently supported by the data. Cross-price responses are also regularized through a soft elasticity-band penalty, which discourages unusually large cross-product effects unless they are supported by the demand-fit term. Because these modeling choices may influence the resulting elasticity distributions, we also examine the robustness of the empirical findings to alternative economic constraints and regularization settings.

Our contributions are:
\begin{itemize}

    \item We make explicit the row-wise integrability and path independence induced by demand-first elasticity estimation: log-demand is represented as a common smooth surface, and own- and cross-price elasticities are obtained as derivatives of that surface.
    
    
    \item We present ICDN, a structured neural demand model designed to make this demand-first principle practical in high-dimensional retail settings. ICDN decomposes price responses into linear own-price effects, directed cross-price interactions, and spline-based nonlinear refinements, allowing flexible but analytically tractable demand surfaces.
    
    \item We exploit analytic derivatives of product-specific spline bases to compute own- and cross-price elasticities without relying on dense generic Jacobian or Hessian evaluations. This makes the derivative calculations transparent and scalable for multiproduct panels.
    
    \item We introduce SKU-specific contextual conditioning together with sparse attention-modulated directed interactions. Each product is encoded through a shared contextual encoder, while cross-price effects are constructed from ordered product pairs, yielding heterogeneous and directional cross-price responses.

    \item We evaluate ICDN against regularized and unregularized log-log benchmarks and a feature-matched demand-first MLP across three retail datasets: Dominick, Walmart M5, and Dunnhumby.

    \item We conduct component-ablation, economic-constraint sensitivity, targeted architectural-substitution, and computational-scalability analyses to characterize the empirical role and trade-offs of the main ICDN design choices.
\end{itemize}

\subsection{Related Work}

The literature most closely related to ICDN can be organized into five complementary streams: traditional demand models, machine-learning and deep-learning demand prediction, context-dependent demand modeling, elasticity estimation from flexible models, and integrability-based demand approaches. We review these streams separately to clarify ICDN’s positioning and methodological gap. 

\paragraph{Traditional demand models.}
Classical structural demand systems specify a system of demand equations and recover elasticities from the estimated demand representation. Within this tradition, the Almost Ideal Demand System (AIDS) of \cite{DeatonMuellbauer1980} models budget shares across goods as functions of prices and real total expenditure. It provides a flexible first-order approximation to consumer demand and a theoretical framework in which standard demand restrictions---including adding-up, homogeneity, and Slutsky symmetry---can be imposed and tested. AIDS was introduced as a tractable alternative to earlier systems such as the Rotterdam model (\cite{Theil1992}) and translog specifications (\cite{Christensen1975TranscendentalLU}). Quadratic extensions such as QUAIDS further increase flexibility by allowing expenditure shares to be quadratic in log expenditure while preserving a utility-consistent structure for analyzing price responses and welfare effects (\cite{BanksBlundellLewbel1997}).

In parallel, the differentiated-products literature, exemplified by \cite{Berry1995}, develops structural models of demand and supply in oligopolistic markets, where substitution patterns arise from heterogeneous consumer preferences over product attributes. In these models, elasticities are not estimated as independent parameters but are recovered from an estimated system of consumer demand. ICDN shares this demand-centered perspective, although it does not impose the adding-up, homogeneity, Slutsky-symmetry, or utility-maximization restrictions associated with classical structural demand systems.

A related applied tradition estimates price responsiveness using log-log demand specifications. These models are widely used in pricing and retail analytics because they are computationally inexpensive and directly interpretable: the coefficient on a product's own log-price represents its own-price elasticity, while coefficients on other products' log-prices represent cross-price elasticities. 
Their main limitation is that they impose globally constant, or otherwise highly restricted, price responses and therefore struggle to represent nonlinear, heterogeneous, and state-dependent substitution patterns in multiproduct retail environments (\cite{Baddeley_Barrowclough_2009}).

\paragraph{Machine-learning and deep-learning demand prediction.}
More recently, demand estimation and forecasting have increasingly adopted machine-learning methods designed for high-dimensional environments. These methods can accommodate nonlinearities, complex interactions, large sets of product descriptors, and rich temporal covariates without requiring a fully specified parametric demand system. \cite{BajariEtAl2015DemandML}, for example, discuss the use of modern machine-learning methods for flexible demand estimation. In a retail application, \cite{Antipov2020InterpretableMLDemand} consider nearly 500 predictors and show that Gradient Boosting Machines outperform Random Forest and Elastic Net alternatives, while Shapley-value decompositions provide interpretable feature attributions. Similarly, \cite{FerreiraLeeSimchiLevi2016} combine machine-learning demand prediction with retail price optimization.

Neural forecasting models extend this predictive perspective by pooling information across related time series and incorporating static and time-varying covariates. Representative examples include the recurrent probabilistic forecasting architecture DeepAR (\cite{SalinasEtAl2020DeepAR}) and the attention-based Temporal Fusion Transformer (\cite{LimEtAl2021TFT}). 
Their primary objective, however, is accurate demand forecasting. Even when price is included as an input, the architecture is not generally designed to make price derivatives explicit, economically regularized, or analytically decomposable into own- and cross-price effects.


ICDN is related to this literature through its use of shared neural representations and high-dimensional contextual information, but maintains an explicit demand-as-a-function-of-price representation whose coefficients are generated by the neural network. Elasticities are therefore intrinsic derivatives of the fitted demand surface rather than post-hoc feature attributions.

\paragraph{Context-dependent demand modeling.}
A complementary literature allows regression effects to vary with observed characteristics or market states. Varying-coefficient models generalize fixed-coefficient regressions by allowing the effect of a predictor to change as a function of one or more modifying variables (\cite{HastieTibshirani1993,FanYaoCai2003}). In demand applications, related forms of heterogeneity arise through product-specific preferences, market-level variation, and flexible interactions between product attributes and contextual covariates (\cite{Berry1995,BajariEtAl2015DemandML}).

ICDN belongs to this tradition, but context does not enter only as an additional predictor of the demand level. Instead, product, store, temporal, promotional, and competitive information is encoded into SKU-specific latent representations that generate the coefficients attached to the explicit price-response terms. Consequently, the complete price-demand surface, including its own- and cross-price derivatives, can vary across products and observed market states.

\paragraph{Elasticity estimation from flexible demand models.}
Flexible learning methods can be used for elasticity estimation through two conceptually different strategies. A first approach estimates an elasticity-related quantity directly. For example, \cite{GREENSTEINMESSICA2020100914} combine machine learning and optimization to infer price responsiveness for products without a sufficiently informative elasticity history. Although such approaches may produce useful local price-response estimates, independently parameterized elasticities need not correspond to the derivatives of a single coherent multiproduct demand surface.

A second strategy is demand-first: a differentiable demand or choice model is estimated first, and price elasticities are subsequently obtained by differentiating its predictions. \cite{AcunaAgost2021PriceElasticityDeepLearning} follow this approach and recover price elasticities from deep-learning-based choice models using automatic differentiation. 
This shares ICDN’s demand-first view: elasticities are derived from a differentiable predictor rather than estimated independently.

The distinction lies in the fitted structure and derivative computation: in a generic neural demand model, price effects may remain distributed across many hidden nonlinear transformations. Recovering large elasticity matrices or higher-order derivative quantities can consequently require repeated generic Jacobian or Hessian operations. ICDN instead separates contextual representation from an explicit spline-based price branch. This makes the contribution of each price component visible and provides closed-form expressions for the own- and cross-price elasticities and for the curvature quantities used during regularization.

\paragraph{Integrability and demand-first approaches.}
The classical literature on demand integrability studies the conditions under which a system of demand functions can be rationalized as the outcome of utility maximization. Under suitable regularity conditions, rationalizable Marshallian demand satisfies restrictions including Walras' law, homogeneity of degree zero, and symmetry and negative semidefiniteness of the Slutsky substitution matrix (\cite{HurwiczUzawa1971}). Complementary nonparametric revealed-preference results characterize the rationalizability of finite demand observations through conditions such as GARP (\cite{Varian1982}). 
These notions impose more economic structure than the row-wise coherence considered here: ICDN does not require utility rationalizability, cross-equation Slutsky symmetry, or the other restrictions of an integrable consumer demand system.

Our use of integrability refers instead to the compatibility of each row of the elasticity matrix with a single corresponding log-demand component. When elasticities are defined as derivatives of a sufficiently smooth log-demand function, the associated elasticity one-forms are exact and the corresponding price responses are path-independent. This demand-first property is not exclusive to ICDN: it is shared by any sufficiently smooth demand model whose elasticities are obtained by differentiating its predictions. 
ICDN’s contribution is therefore not this general mathematical result, but its implementation in a structured multiproduct architecture with context-dependent coefficients, explicit nonlinear price bases, sparse directed interactions, analytic derivatives, and derivative-oriented regularization.

\paragraph{Empirical evidence on beer demand.}
Our empirical application also connects to the literature on price elasticities for fast-moving consumer goods and beer demand. At an aggregate level, the evidence generally indicates negative and inelastic beer demand. \cite{NELSON2014180} provides a meta-analysis of 191 beer price-elasticity estimates from 114 primary studies and report an average elasticity of approximately $-0.20$ after accounting for heterogeneity, dependence, and publication bias. Similarly, \cite{ClementsMarianoVerikiosWong2022} report a mean of means of $-0.47$ and a mean of medians of $-0.37$ across relevant studies.

Estimated elasticities also depend strongly on the level of product aggregation. Using UPC scanner information, \cite{RuhmEtAl2012} obtain a relatively low beer-category elasticity of approximately $-0.3$, although UPC observations are aggregated into category-level prices. By contrast, \cite{Srivastava2014} estimate an AIDS model for 12 disaggregated alcoholic-beverage types and obtain substantially larger own-price elasticities, with beer-type estimates ranging approximately from $-1.0$ to $-5.4$.

Closest to our empirical setting, \cite{ToroGonzalezMcCluskeyMittelhammer2014} use Dominick's supermarket scanner data to estimate demand for differentiated beer products. Although the underlying data contain UPC-level information, their reported elasticities are aggregated into mass-produced, craft, and imported beer categories. They obtain own-price elasticities of approximately $-0.13$, $-0.21$, and $-0.22$, respectively, and cross-price effects across beer types that are close to zero. This evidence motivates the empirical prior that own-price responses should be predominantly negative, whereas cross-price responses are likely to be weaker, more heterogeneous, and more difficult to estimate reliably. ICDN moves beyond beer-type aggregates by estimating model-implied context-dependent own- and cross-price elasticities for individual SKUs/UPCs.

\paragraph{Positioning of ICDN.}
Classical structural and log-log demand models preserve a transparent connection between demand and elasticities but generally rely on restrictive price-response specifications or stronger economic assumptions. Machine-learning and deep-learning models provide substantially greater predictive flexibility and can incorporate rich retail context, but they are commonly designed primarily for demand prediction rather than for transparent and stable elasticity recovery. Context-dependent and varying-coefficient models allow responses to vary across products and market states, while derivative-based neural approaches recover elasticities from flexible predictors. However, these approaches do not generally combine explicit price effects, context-conditioned coefficients, analytic price derivatives, sparse directed multiproduct interactions, and derivative-oriented regularization within a single architecture.

ICDN is positioned at the intersection of these streams. Its novelty does not lie in neural networks, spline bases, attention mechanisms, or demand-first integrability individually. Rather, it lies in integrating these established elements into a structured multiproduct demand architecture in which own- and cross-price effects remain explicit, nonlinear, context-dependent, direction-specific, sparsely evaluated, analytically differentiable, and directly regularized.


\section{Elasticity Fields and Integrability}\label{Sec:elasticity-path}

In multiproduct pricing, demand is typically modeled as a function of prices and contextual information, while elasticities provide a compact summary of local price responses. Let $q(p,x)\in\mathbb{R}^n_{+}$ denote the demand vector at prices $p\in\mathbb{R}^n_{+}$ and contextual covariates $x$. The (point) price elasticity matrix is defined componentwise as
\[
E_{ij}(p,x)\;=\;\frac{\partial \log q_i(p,x)}{\partial \log p_j},
\]
so that $E_{ij}$ measures how the demand of product $i$ changes under an infinitesimal multiplicative perturbation of the price of product $j$. Elasticities are attractive in practice because they are scale-free, comparable across SKUs, and directly usable for reporting and scenario evaluation.

At the same time, $E(p,x)$ summarizes local derivatives, not the demand quantities $q(p,x)$ themselves. This raises a basic consistency question in multiple dimensions: when does an elasticity field $E(\cdot,x)$ correspond to a single, globally well-defined demand function?

The term elasticity has several related meanings, including point and arc elasticities, own- and cross-price elasticities, Marshallian and Hicksian elasticities, and elasticity and semi-elasticity interpretations in logarithmic models\footnote{For further discussion of these distinctions, see \cite{poitras_point_elasticity} and \cite{cameron_aed_logs}.}. In this paper, we focus on point log-price elasticities of observed demand conditional on contextual covariates, evaluated holding the context fixed. We do not interpret cross-price terms as Hicksian compensated effects, nor do we impose Slutsky symmetry. Instead, cross-price elasticities are interpreted as local directional responses in the fitted context-conditioned demand surface: \(E_{ij}\) need not coincide with the reverse effect \(E_{ji}\). We now define this elasticity matrix formally.

\paragraph{Notation and elasticity matrix.}
Let \(p=(p_1,\dots,p_n)\in\mathbb{R}_{++}^n\) be the vector of strictly positive effective prices for \(n\) SKUs, and let \(x\in\mathcal{X}\) denote contextual covariates, such as store and SKU identity, time features, and promotion indicators. Demand, measured in volume, is modeled as a continuously differentiable map $v:\mathbb{R}_{++}^n\times\mathcal{X}\to\mathbb{R}_{++}^n$, $\,(p,x)\mapsto v(p,x)$, with components \(v_i(\cdot,x)\in C^1(\mathbb{R}_{++}^n)\) for each fixed \(x\). Throughout this section, derivatives are taken with respect to prices while holding the context \(x\) fixed. 

\begin{definition}\label{def:elasticity}
The point price elasticity matrix \(E(p,x)\in\mathbb{R}^{n\times n}\) is defined componentwise by
\[
E_{ij}(p,x)
:=
\frac{\partial \log v_i(p,x)}{\partial \log p_j}
=
\frac{p_j}{v_i(p,x)}
\frac{\partial v_i}{\partial p_j}(p,x),
\qquad 1\le i,j\le n.
\]
Equivalently, using elementwise logarithms,
\[
d\big(\log v(p,x)\big)
=
E(p,x)\,d\big(\log p\big),
\]
where the differential is taken with respect to \(p\), holding \(x\) fixed.
\end{definition}

In the one-dimensional case, \(n=1\), the elasticity matrix reduces to a scalar elasticity $\varepsilon(p,x)= d\log v(p,x) /d\log p,$ and, for each fixed context \(x\), the elasticity relation can be integrated directly with respect to \(p\). Given a baseline pair \((p_0,v(p_0,x))\) with \(p_0>0\) and \(v(p_0,x)>0\), demand at any \(p>0\) is given by
\[
v(p,x)
=
v(p_0,x)
\exp\!\left(
\int_{p_0}^{p}
\varepsilon(s,x)\frac{ds}{s}
\right).
\]
Thus, in one dimension, specifying the local elasticity function \(\varepsilon(\cdot,x)\) determines the entire demand curve up to the positive normalization \(v(p_0,x)\).

In the multidimensional case, \(n>1\), the elasticity field \(E(\cdot,x)\) specifies local derivatives of log-demand in several price directions. This introduces an additional consistency requirement that is absent in the one-dimensional case. For each SKU \(i\) and fixed context \(x\), define the \(1\)-form on \(\mathbb{R}_{++}^n\)
\[
\omega_i(p,x)
:=
\sum_{j=1}^n E_{ij}(p,x)\,d(\log p_j).
\]
If the elasticity field is induced by a differentiable demand function, then Definition~\ref{def:elasticity} can be written componentwise as $d\big(\log v_i(p,x)\big) = \omega_i(p,x),$  $i=1,\dots,n$. The question is therefore whether the local responses encoded in \(\omega_i\) can be integrated into a single log-demand function \(\log v_i(\cdot,x)\). 

\begin{proposition}\label{prop:integrability-elasticity}
Fix a context value \(x\in\mathcal X\) and a demand component \(i\in\{1,\dots,n\}\). Let \(U\subset \mathbb{R}_{++}^n\) be a simply connected open set and assume that \(E_{ij}(\cdot,x):U\to\mathbb{R}\) is of class \(C^1\) for every \(j=1,\dots,n\). Define the \(1\)-form
\[
\omega_i(p,x)
:=
\sum_{j=1}^n E_{ij}(p,x)\,d(\log p_j),
\qquad p\in U.
\]
If \(\omega_i(\cdot,x)\) is closed on \(U\), that is, if
\[
d\omega_i(\cdot,x)=0
\qquad \text{on } U,
\]
then there exists a function \(v_i(\cdot,x):U\to(0,\infty)\), unique up to a positive multiplicative constant, such that
\[
d\big(\log v_i(p,x)\big)
=
\omega_i(p,x),
\qquad
p\in U.
\]
Equivalently, the closure condition can be written as the commutation of cross-partials in log-price coordinates:
\[
\partial_{\log p_k}E_{ij}(p,x)
=
\partial_{\log p_j}E_{ik}(p,x),
\qquad
\forall\, j,k\in\{1,\dots,n\},\ \forall\, p\in U.
\]
\end{proposition}

\begin{proof}[Proof sketch]
Let \(u=\log p\) and \(V=\log(U)\). Under this change of coordinates, the elasticity \(1\)-form becomes
\[
\widetilde{\omega}_i(u,x)
=
\sum_{j=1}^{n}E_{ij}(e^u,x)\,du_j,
\]
which is closed on the simply connected set \(V\). By the Poincaré lemma, a closed \(1\)-form on a simply connected domain is exact (\cite{choquet1982analysis}, \textit{Poincaré Lemma}, pag. 224). Hence, there exists a potential \(\psi_i(\cdot,x)\) such that $d\psi_i(u,x)=\widetilde{\omega}_i(u,x)$. Defining $v_i(p,x)=\exp\!\bigl(\psi_i(\log p,x)\bigr)$ gives \(v_i(p,x)>0\) and $d\log v_i(p,x)=\omega_i(p,x)$.

If \(v_i\) and \(\widetilde v_i\) are two positive solutions, then \(d\log(v_i/\widetilde v_i)=0\). Since \(U\) is connected, \(v_i/\widetilde v_i\) is a positive constant.
\end{proof}

Proposition~\ref{prop:integrability-elasticity} links an elasticity field to a globally defined demand map. The closure condition \(d\omega_i=0\) ensures that the local elasticities for product $i$ can be integrated into a single log-demand function. Integrability is therefore a mathematical coherence requirement, not an additional economic assumption.


To see the implication, fix a baseline \(p_0\in U\) and let \(\gamma:[0,1]\to U\) be any piecewise smooth path with \(\gamma(0)=p_0\) and \(\gamma(1)=p\). If \(\omega_i\) is exact, then
\[
\log v_i(p,x)-\log v_i(p_0,x)
=
\int_\gamma \omega_i(\tilde p,x)
=
\sum_{j=1}^n
\int_\gamma
E_{ij}(\tilde p,x)\,d(\log \tilde p_j),
\]
and therefore
\[
v_i(p,x)
=
v_i(p_0,x)
\exp\!\left(
\int_\gamma \omega_i(\tilde p,x)
\right).
\]
Because exact forms have path-independent line integrals, the predicted demand \(v_i(p,x)\) depends only on the final price vector \(p\) and on the normalization \(v_i(p_0,x)\), not on the particular sequence of intermediate price changes used to reach \(p\). See Appendix \ref{Appendix:Constant_Elasticity} for an example of a constant elasticity matrix. 

If the closure condition fails, this coherence breaks down. Two different price-update paths ending at the same price vector can yield different values of $\int_\gamma \omega_i$, and therefore different reconstructed demands for the same final prices. One could restore determinism by imposing a canonical update path, for example by changing SKUs in a prescribed order, but this would make the forecast depend on an arbitrary convention rather than on a coherent demand function.

This distinction matters for learning-based elasticity estimation. If a model predicts elasticities directly rather than deriving them from demand, the resulting field need not satisfy \(d\omega_i = 0\) and may therefore be path-dependent or incompatible with any global demand map (\cite{GREENSTEINMESSICA2020100914}).


By contrast, if one models log-demand directly and defines elasticities by differentiation, $E_{ij}(p,x)=\partial_{\ln p_j}\ln v_i(p,x)$, then $\omega_i=d\ln v_i$ holds by construction, and hence \(d\omega_i = 0\) follows automatically under the required smoothness conditions. 
Because this applies to any sufficiently smooth demand-first model, we treat integrability as a demand-first design principle rather than an ICDN-specific contribution.


\section{Elasticities with Neural Networks}\label{sec:nn}

Section~\ref{Sec:elasticity-path} shows that coherent multiproduct elasticities must be compatible with a single log-demand surface. We use this demand-first principle to construct a neural model that learns a flexible and sufficiently smooth context-conditioned surface whose derivatives yield stable and economically interpretable elasticities. 


Let \(p\in\mathbb{R}_{++}^n\) denote the vector of effective prices and let \(x\in\mathcal X\) denote contextual covariates, including store and SKU identity, seasonality, promotion indicators, and other market-state variables. We work in log-price coordinates $u=\log p\in\mathbb{R}^n$, and define log-demand as $y(u,x):= \log v(\exp(u),x)\in\mathbb{R}^n$. Context \(x\) affects the entire price-demand surface, not only the baseline demand level. Consequently, the same price vector can imply different elasticities across stores, products, seasons, and promotional regimes.

The elasticity matrix is therefore a local function of prices and context. It is given by the Jacobian of log-demand with respect to log-prices, that is, the matrix containing all local own- and cross-price derivatives:
\begin{equation}
E(u,x)
:=
\frac{\partial y(u,x)}{\partial u}
\in\mathbb{R}^{n\times n},
\qquad
E_{ij}(u,x)
=
\partial_{u_j}\log v_i(\exp(u),x).
\end{equation}
Since \(u_j=\log p_j\), this is equivalent to Definition~\ref{def:elasticity}. Thus, \(E_{ij}(u,x)\) measures the local response of product $i$'s demand to the log-price of product $j$ at the price vector and context \((u,x)\).

We parameterize log-demand with a neural model
\begin{equation}
y(u,x)
=
g_\theta(u,x)
\approx
\log v(\exp(u),x),
\end{equation}
where \(g_\theta\) represents the fitted context-conditioned log-demand surface. Demand predictions are therefore obtained by evaluating this surface directly at the target price vector and context. Elasticities are not modeled as separate outputs; instead, they are defined as the log-price derivatives of the same fitted surface:
\begin{equation}\label{eq:deriv_g_theta}
E(u,x)
=
J_u g_\theta(u,x),
\qquad
E_{ij}(u,x)
=
\partial_{u_j} g_{\theta,i}(u,x).
\end{equation}
Because all elasticities are derivatives of the same surface, the integrability property of Proposition~\ref{prop:integrability-elasticity} is preserved. If \(g_{\theta,i}(\cdot,x)\) is twice continuously differentiable in log-prices, the induced row-wise elasticity 1-form satisfies
\[
\omega_i(u,x)
=
\sum_{j=1}^n
\partial_{u_j}g_{\theta,i}(u,x)\,du_j
=
d g_{\theta,i}(u,x).
\]
Hence, the form is exact and the closure conditions hold automatically. Section~\ref{sec:nn-method} formalizes this result.

Building on this formulation, we introduce the Integrable Context-Dependent Demand Network (ICDN), a structured neural demand model for multiproduct retail panels, where price effects may be nonlinear, context-dependent, and sparse across products. Unlike a generic neural predictor, ICDN keeps price effects explicit and elasticities analytically tractable.


The explicit price representation also makes derivative calculations tractable. Because the smooth price basis has closed-form derivatives, ICDN avoids materializing dense Jacobians or Hessians and computes only the quantities required for training, regularization, and elasticity analysis. The architecture therefore retains nonlinear and context-dependent price responses while preserving derivative coherence and computational tractability.


Integrability is not specific to ICDN: any sufficiently smooth demand-first model induces a row-wise integrable elasticity field when elasticities are derived from predicted log-demand. ICDN implements this principle through a structured log-price representation combining linear effects, spline-based nonlinear refinements, and context-dependent directed cross-price interactions.


\subsection{The icdn demand structure}
\label{sec:nn-method}

We now specify the ICDN demand structure. 
Predicted demand is $\widehat v_\theta(p,x) =\exp\!\big(g_\theta(\log p,x)\big)$ , where the exponential is applied componentwise; hence, it is positive by construction. As established in Eq.~\ref{eq:deriv_g_theta}, the model-implied elasticities are given by $\widehat E_{\theta,ij}(u,x) = \partial_{u_j}g_{\theta,i}(u,x).$

ICDN parameterizes \(g_\theta\) by combining an explicit price representation with context-dependent coefficients. Its nonlinear price dependence is represented through product-specific spline bases. All products share the same spline family and number \(K\) of basis functions, but their knots adapt to each SKU's empirical price support. 
Let \(Q_i\) be the empirical quantile function of the training observations of \(u_i\). Given \(K\) spline basis functions, we choose uniformly quantile points\footnote{We use the interior quantile range $[0.05,0.95]$ rather than the full $[0,1]$ range to avoid placing knots at extreme price observations.}
\[
\tau_1<\cdots<\tau_K,
\qquad
\tau_k\in[0.05,0.95],
\]
and define the product-specific knots in the original log-price scale as
\begin{equation}\label{eq:knots}
 \tilde\kappa_{ik}
=
Q_i(\tau_k),
\qquad
k=1,\dots,K.   
\end{equation}
We also compute a product-specific centering and scaling,
\begin{equation}\label{eq:normalization}
\mu_i
=
\mathbb{E}_{\mathrm{train}}[u_i],
\qquad
\sigma_i
=
\max\{\operatorname{sd}_{\mathrm{train}}(u_i),\,0.2\},  
\end{equation}
where the lower bound on \(\sigma_i\) avoids numerical instability for products with very limited price variation. Together, the quantile-based knot placement and the product-specific normalization ensure that the spline basis is both empirically localized and numerically comparable across SKUs. Using the normalized coordinates \(z_i=(u_i-\mu_i)/\sigma_i\) and \(\kappa_{ik}=(\widetilde\kappa_{ik}-\mu_i)/\sigma_i\), the spline basis \(B_i(u_i)=\bigl(B_{i1}(u_i),\ldots,B_{iK}(u_i)\bigr)^\top \in\mathbb R^K\) has components
\[
B_{ik}(u_i)
=
[\operatorname{ReLU}(z_i-\kappa_{ik})]^3
=
\left[
\operatorname{ReLU}
\left(\frac{u_i-\widetilde\kappa_{ik}}{\sigma_i}\right)
\right]^3.
\]
The knots, centering terms, and scales are fixed before training; the network learns only the context-dependent coefficients multiplying these basis functions.

Unlike a global quadratic term such as \(u_i^2\), truncated-power splines introduce localized curvature: each term activates after a product-specific knot, allowing slope and curvature to vary across empirically supported price regions. Figure~\ref{fig:quadratic-vs-splines} illustrates this distinction.
\begin{figure}[htbp]
    \centering

    \begin{subfigure}{0.43\textwidth}
        \centering
        \includegraphics[width=\textwidth]{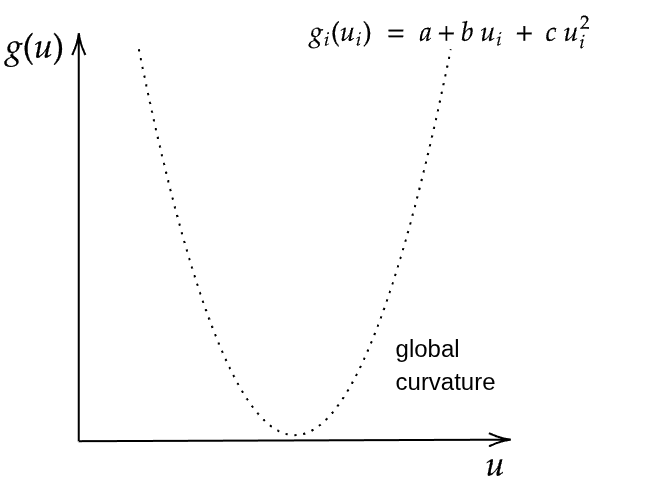}
        \caption{Global polynomial curvature.}
        \label{fig:quadratic-curvature}
    \end{subfigure}
    \hfill
    \begin{subfigure}{0.48\textwidth}
        \centering
        \includegraphics[width=\textwidth]{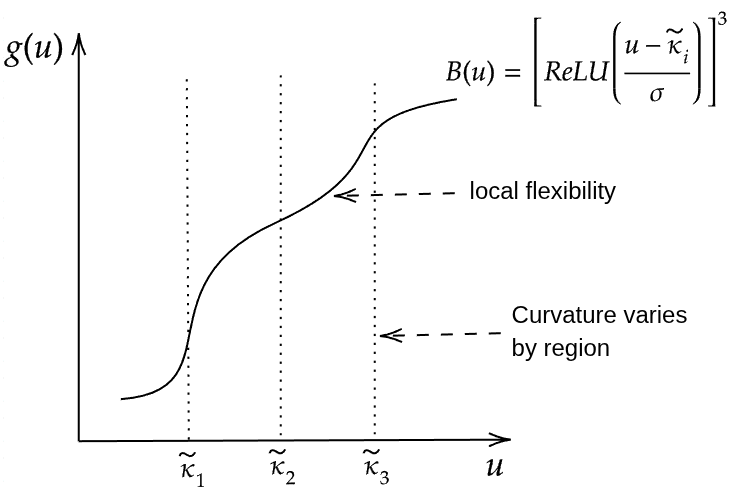}
        \caption{Local spline flexibility.}
        \label{fig:spline-flexibility}
    \end{subfigure}

    \caption{
    Comparison between a global quadratic price effect and a spline-based nonlinear price representation. A quadratic term such as \(u_i^2\) imposes a single curvature pattern over the full log-price domain. In contrast, truncated-power spline terms of the form \(B_{ik}(u_i)=[\operatorname{ReLU}((u_i-\tilde\kappa_{ik})/\sigma_i)]^3\) activate after product-specific knots, allowing the slope and curvature of the fitted log-demand surface to vary across price regions while preserving differentiability.
    }
    \label{fig:quadratic-vs-splines}
\end{figure}

The cubic order balances local flexibility with the smoothness required for integrability. Since \([\operatorname{ReLU}(z_i-\kappa_{ik})]^m\in C^{m-1}\), \(m=3\) is the lowest order in this family that yields a twice continuously differentiable (\(C^2\)) basis, ensuring that elasticity derivatives are well defined.


Since the basis is evaluated in normalized coordinates, its derivatives with respect to the original log-price \(u_i\) are available in closed form. In particular,
\begin{equation}\label{eq:splines}
B_{ik}^{(r)}(u_i)
=
\frac{3!}{(3-r)!}\,
\frac{1}{\sigma_i^r}
\operatorname{ReLU}
\left(
\frac{u_i-\tilde\kappa_{ik}}{\sigma_i}
\right)^{3-r},
\qquad r\in\{0,1,2\}.
\end{equation}
Thus, nonlinear price effects remain flexible while their elasticity and curvature contributions are analytically tractable.

For each demand equation \(i\), the context encoder produces a baseline coefficient \(b_i(x)\in\mathbb R\) and, for each ordered demand-price pair \((i,j)\), price-response coefficients \(\beta_{ij}(x)\in\mathbb R\) and \(w_{ij}(x)\in\mathbb R^K\).
The coefficient \(\beta_{ij}(x)\) captures the locally linear effect of the log-price of product \(j\) on the log-demand of product \(i\), while \(w_{ij}(x)\) controls a product-\(j\)-specific nonlinear spline refinement of that effect. The case \(j=i\) corresponds to the own-price response of product \(i\), whereas \(j\neq i\) corresponds to directed cross-price responses.

In the implementation, the own-price linear coefficient is sign-constrained. The network outputs an unconstrained scalar \(\beta^{\mathrm{raw}}_{ii}(x)\), which is transformed as
\begin{equation}\label{eq:beta_ii}
 \beta_{ii}(x) = -\operatorname{softplus}\!\left(\beta^{\mathrm{raw}}_{ii}(x)\right).   
\end{equation}
This makes the linear own-price component non-positive by construction, consistent with the standard consumer-demand prior that own-price responses should be non-positive. The constraint applies only to the linear own-price coefficient; cross-price coefficients \(\beta_{ij}(x)\), \(i\neq j\), remain unrestricted.

For each ordered pair \(i\neq j\), the model also produces \(U^{(ij)}(x)\in\mathbb R^{K\times K}\), which governs the spline-spline interaction between the prices of products \(i\) and \(j\) in product \(i\)'s demand equation. This allows cross-price effects to vary across pricing regimes; for example, a discount on \(j\) may cannibalize \(i\) differently depending on \(i\)'s price position. Because pairs are ordered, the effect of \(j\) on \(i\) is learned separately from the reverse effect.

Taken together, the \(i\)-th component of the fitted log-demand surface is defined as
\begin{equation}
g_{\theta,i}(u,x)
=
b_i(x)
+
\sum_{j=1}^{n}
\left[
\beta_{ij}(x)u_j
+
w_{ij}(x)^\top B_j(u_j)
\right]
+
\sum_{j\neq i}
B_i(u_i)^\top U^{(ij)}(x)B_j(u_j).
\label{eq:icdn-g}
\end{equation}
In Eq.~\ref{eq:icdn-g}, the first summation contains the direct own- and cross-price terms, each combining a linear effect and a nonlinear spline refinement. The second contains pairwise interactions between the price positions of products \(i\) and \(j\). Because all coefficients depend on \(x\), these responses may vary across products and market contexts.

Elasticities follow by differentiating Eq.~\ref{eq:icdn-g} with respect to log-prices. For \(i\neq j\), the cross-price elasticity of product \(i\) with respect to the price of product \(j\) is
\[
\widehat E_{\theta,ij}(u,x)
=
\partial_{u_j}g_{\theta,i}(u,x)
=
\beta_{ij}(x)
+
w_{ij}(x)^\top B_j'(u_j)
+
B_i(u_i)^\top U^{(ij)}(x)B_j'(u_j).
\]
The own-price elasticity is
\[
\widehat E_{\theta,ii}(u,x)
=
\partial_{u_i}g_{\theta,i}(u,x)
=
\beta_{ii}(x)
+
w_{ii}(x)^\top B_i'(u_i)
+
\sum_{j\neq i}
B_i'(u_i)^\top U^{(ij)}(x)B_j(u_j).
\]
Thus, each element of \(\widehat E_\theta(u,x)\) is available in closed form\footnote{The closed-form elasticities can also support price-region diagnostics. However, evaluating pricing decisions requires an explicit revenue, margin, or profit objective and is beyond the scope of this paper.} from the product-specific spline bases \(B_j\), their derivatives \(B_j'\), and the context-dependent coefficients. Furthermore, although \(\beta_{ii}(x)\) is non-positive, the total own-price elasticity also contains nonlinear spline and interaction derivatives. The sign constraint therefore does not guarantee global monotonicity, and local positive own-price elasticities remain possible.

Since the elasticity matrix is induced by differentiating the fitted demand potential \(g_\theta\), the closure conditions from Proposition~\ref{prop:integrability-elasticity} hold by construction. The following proposition formalizes this property.

\begin{proposition}
\label{prop:integrability-by-construction}
Fix \(x\in\mathcal X\) and let \(V\subset\mathbb R^n\) be open. If
\(g_{\theta,i}(\cdot,x)\in C^2(V)\), then the elasticity \(1\)-form
\[
\widehat\omega_{\theta,i}(u,x)
=
\sum_{j=1}^n
\partial_{u_j}g_{\theta,i}(u,x)\,du_j
\]
is exact and therefore closed. Equivalently,
\[
\partial_{u_k}\widehat E_{\theta,ij}(u,x)
=
\partial_{u_j}\widehat E_{\theta,ik}(u,x),
\qquad
\forall j,k,\quad \forall u\in V.
\]
\end{proposition}


\begin{proof}
Since \(\widehat\omega_{\theta,i}=d g_{\theta,i}\), it is exact and hence closed. The equality of mixed partial derivatives follows from Clairaut's theorem.
\end{proof}

Proposition~\ref{prop:integrability-by-construction} is a general coherence result of demand-first modeling, not an ICDN-specific theoretical contribution. ICDN's distinction lies in its explicit spline-based price representation, analytic derivatives, sparse directed interactions, and derivative-oriented regularization.


\subsection{Architecture overview}\label{sec:architecture}

Figure~\ref{fig:icdn-arch} summarizes the ICDN forward pass. A context branch produces one latent representation per SKU, a sparse neighbor selector identifies directed product interactions, and a price branch evaluates the spline basis and its analytic derivatives. The demand module combines these components to produce log-demand predictions and elasticities.


Rather than compressing the market state into a single global vector, ICDN constructs one contextual token per SKU. The token \(\tau_i(x)\in\mathbb R^{d_{\mathrm{tok}}}\) combines global market information with SKU-specific features. Stacking the tokens gives \(\tau(x)=(\tau_1(x),\ldots,\tau_n(x)) \in\mathbb R^{n\times d_{\mathrm{tok}}}\). A shared encoder then maps each token to a SKU-specific latent representation, \(h_i=\phi_\theta(\tau_i(x))\), yielding \(h(x)=(h_1(x),\ldots,h_n(x))\in\mathbb R^{n\times d_h}\).

The tokenized representation separates product encoding from pair construction. Own-price terms are generated from \(h_i\), whereas cross-price terms are generated from ordered pairs \((h_i,h_j)\). This allows the model to evaluate cross-price parameters and interactions only for the $k$ selected neighbors of each product. Once the interaction support is fixed, these selected-edge computations scale as $O(nk)$ with \(k\ll n\), while candidate scoring remains $O(n^2)$.

Own-price parameters are generated as \(h_i\mapsto\{b_i(h),\beta_{ii}(h),w_{ii}(h)\}\), while cross-price parameters are generated from ordered pairs, $h_{\mathrm{pair}}:=(h_i,h_j)
\mapsto
\left\{
\beta_{ij}(h_{\mathrm{pair}}),
w_{ij}(h_{\mathrm{pair}}),
U^{(ij)}(h_{\mathrm{pair}})
\right\},$ with $i\neq j$. Because the pairs are ordered, the effect of \(j\) on \(i\) is learned separately from the reverse effect, allowing \(\widehat E_{\theta,ij}\neq\widehat E_{\theta,ji}\).


To make the cross-price component scalable, ICDN does not retain all off-diagonal product pairs. It selects relevant product pairs using scaled dot-product attention (\cite{NIPS2017_Attention}). For store-week observation \(b=(s,t)\), the latent representations are projected as \(q_i^{(b)}=W_Qh_i^{(b)}\) and \(k_j^{(b)}=W_Kh_j^{(b)}\). Candidate \(j\) is scored for focal product \(i\) as
\[
s_{ij}^{(b)}
=
\frac{q_i^{(b)\top}k_j^{(b)}}{\sqrt{d_{\mathrm{att}}}}
+\xi_{ij},
\qquad i\neq j.
\]
Here, \(\xi_{ij}\) is a fixed bonus based on SKU metadata, including brand, style, and package-size similarity\footnote{Using query-key projections is more flexible than relying directly on the raw similarity \(h_i^\top h_j\). The latter imposes a more rigid and nearly symmetric notion of similarity.}.

The neighbor selector operates in two stages. The first stage determines the sparse support of directed cross-price interactions, while the second stage assigns observation-specific attention weights on the selected edges. 

During training, the first stage is performed online. For each focal SKU \(i\), the model computes a relevance score \(s_{ij}^{(b)}\) for every candidate SKU \(j\) in each store-week observation \(b\) of the mini-batch. The scores are then aggregated across the store-week observations in the mini-batch \(\mathcal B\),
\[
\bar s_{ij}^{\mathcal B}
=
\frac{1}{|\mathcal B|}
\sum_{b\in\mathcal B}
s_{ij}^{(b)},
\]
where \(|\mathcal B|\) is the number of observations in that mini-batch. Self-pairs are masked out, and the top-ranked candidates according to \(\bar s_{ij}^{\mathcal B}\) are retained as temporary neighbors of \(i\) for the current forward pass. Since a focal SKU can interact with at most \(n-1\) other SKUs, the effective number of retained neighbors is $k_{\mathrm{eff}}=\min\{k,n-1\}$, where 
\(k\) is the prescribed number of neighbors per focal SKU. The resulting online sparse edge set is denoted by
\[
\mathcal P^{\mathcal B}
\subset
\{(i,j):i\neq j\},
\qquad
|\{j:(i,j)\in\mathcal P^{\mathcal B}\}|=k_{\mathrm{eff}}.
\]
After training has converged, the sparse graph used for validation, testing, prediction reporting, and elasticity extraction is frozen. For each temporal split, all reported predictions and model-implied elasticities are computed using this frozen training-split graph, so they do not depend on the composition of the evaluation mini-batch.  Specifically, after loading the final trained checkpoint, ICDN computes the average learned relevance score for each directed product pair over the training split $\mathcal D_{\mathrm{train}}$,
\[
\bar s_{ij}^{\mathrm{train}}
=
\frac{1}{|\mathcal D_{\mathrm{train}}|}
\sum_{b\in\mathcal D_{\mathrm{train}}}
s_{ij}^{(b)}.
\]
The final frozen neighbor graph is then defined as
\[
\mathcal P^\star
=
\left\{
(i,j):
j\in
\operatorname{TopK}_{\ell\neq i}
\bar s_{i\ell}^{\mathrm{train}}
\right\},
\]
where \(\operatorname{TopK}_{\ell\neq i}\) denotes the set of the \(k_{\mathrm{eff}}\) highest-scoring candidates among all \(\ell\neq i\). 

Optionally, same-category products can be prioritized during graph selection, although this restriction is not imposed by default. The ranking score becomes
\[
\widetilde s_{ij}
=
\bar s_{ij}
+
M\mathbf 1\{\operatorname{cat}(i)=\operatorname{cat}(j)\},
\qquad M=10^6.
\]
Here, \(M\) is only a numerical ranking device, not an economic magnitude or attention weight. If fewer than \(k_{\mathrm{eff}}\) same-category candidates are available, the remaining positions are filled according to the original scores \(\bar s_{ij}\).

In the second stage, observation-specific attention weights are computed over the active graph \(\mathcal P\), where \(\mathcal P=\mathcal P^\mathcal B\) during training and \(\mathcal P=\mathcal P^\star\) during evaluation. For each observation \(b\), the original logits \(s_{ij}^{(b)}\), including the metadata bonus but excluding the category-priority ranking constant $M$, are restricted to the selected neighbors and normalized over the candidate set of each focal product:
\[
a_{ij}^{(b)}
=
\frac{\exp(s_{ij}^{(b)})}
{\sum_{\ell:(i,\ell)\in\mathcal P}\exp(s_{i\ell}^{(b)})},
\qquad (i,j)\in\mathcal P.
\]


In parallel, the price branch evaluates \(\{B_i(u_i),B_i'(u_i),B_i''(u_i)\}_{i=1}^n\) from the log-price vector \(u=(u_1,\ldots,u_n)\), as defined in Eq.~\ref{eq:splines}. The first derivatives
enter the elasticity expressions, and the second derivatives enter the own-price curvature penalty.


Combining the selected graph, attention weights, and spline basis, Eq.~\ref{eq:icdn-g} becomes, with the mini-batch index omitted,
\begin{align}
&g_{\theta,i}(u,h;\mathcal P) =
b_i(h)
+
\beta_{ii}(h)u_i
+
w_{ii}(h)^\top B_i(u_i)
\nonumber\\
&+
\sum_{j:(i,j)\in\mathcal P}
a_{ij}(h_{\text{pair}})
\left[
\beta_{ij}(h_{\text{pair}})u_j
+
w_{ij}(h_{\text{pair}})^\top B_j(u_j)
+
B_i(u_i)^\top U^{(ij)}(h_{\text{pair}})B_j(u_j)
\right].
\label{eq:icdn-sparse-g}
\end{align}
\begin{figure}[!p]
  \centering
  \includegraphics[width=0.58\linewidth]{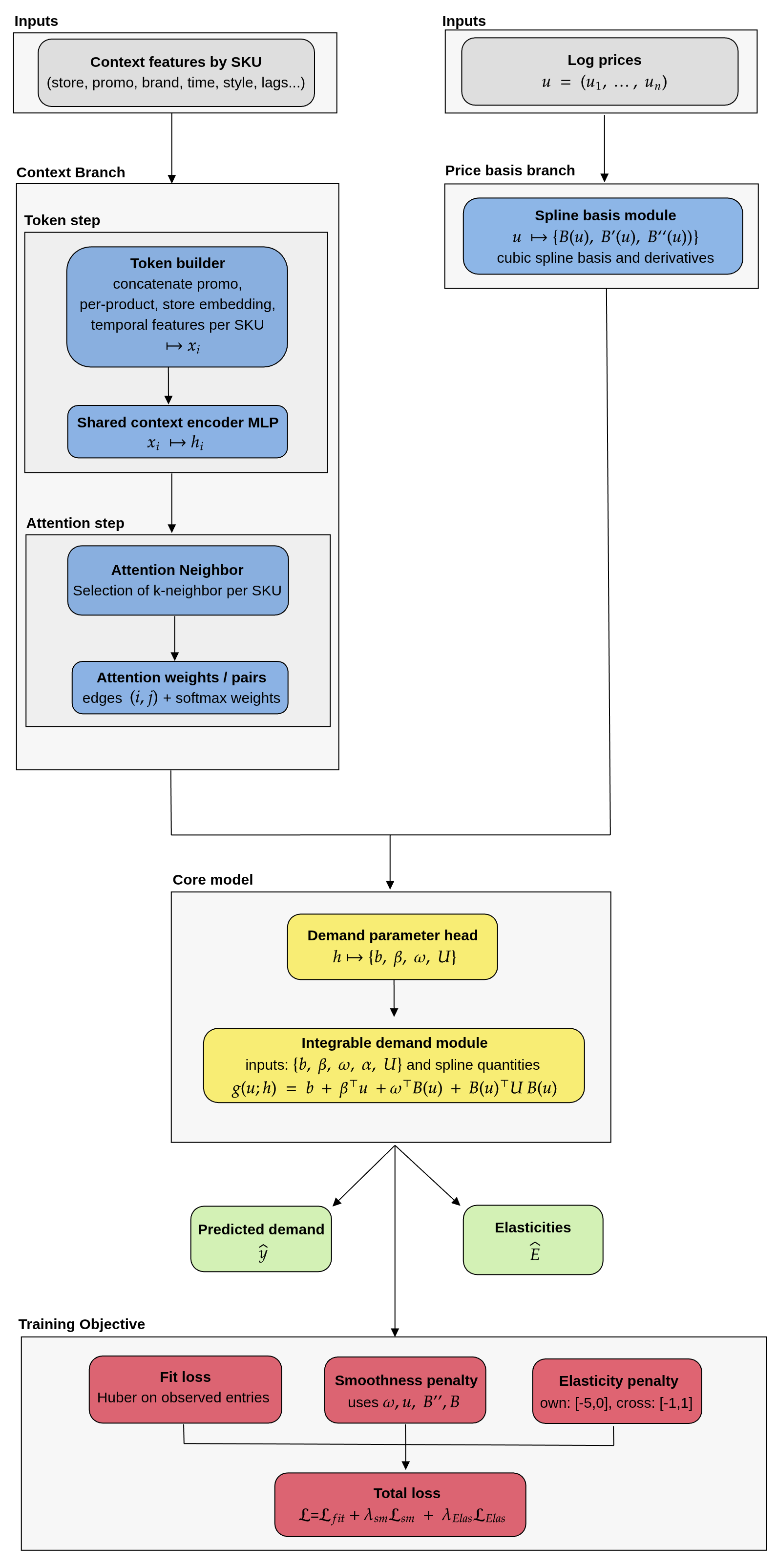}
  \caption{\textbf{ICDN forward pass.} Context tokens are encoded into one latent representation per SKU, which generate own-price response terms, sparse attention-weighted cross-price response terms, and the parameters of the structured demand potential. Spline bases and analytic derivatives of log-prices are combined with these parameters to produce log-demand predictions and elasticities by exact differentiation.}
  \label{fig:icdn-arch}
\end{figure}
Elasticities are obtained by differentiating Eq.~\ref{eq:icdn-sparse-g} with respect to log-prices, holding \(h\), \(\mathcal P\), and \(a_{ij}(h_{\mathrm{pair}})\) fixed. Therefore, the own-price elasticity is
\[
\widehat E_{\theta,ii}(u,h;\mathcal P)
=
\frac{\partial g_{\theta,i}(u,h;\mathcal P)}{\partial u_i}
=
\beta_{ii}(h)
+
w_{ii}(h)^\top B_i'(u_i)
+
\sum_{j:(i,j)\in\mathcal P}
a_{ij}(h_{\text{pair}})
B_i'(u_i)^\top U^{(ij)}(h_{\text{pair}})B_j(u_j).
\]
For each selected pair \((i,j)\in\mathcal P\), \(i\neq j\), the cross-price elasticity is
\[
\widehat E_{\theta,ij}(u,h;\mathcal P)
=
\frac{\partial g_{\theta,i}(u,h;\mathcal P)}{\partial u_j}
=
a_{ij}(h_{\text{pair}})
\left[
\beta_{ij}(h_{\text{pair}})
+
w_{ij}(h_{\text{pair}})^\top B_j'(u_j)
+
B_i(u_i)^\top U^{(ij)}(h_{\text{pair}})B_j'(u_j)
\right].
\]

\subsection{Training objective and regularization}
\label{sec:training-objective}

We train ICDN with a composite objective that balances predictive accuracy, smoothness of the fitted demand surface, and business plausibility of the implied elasticity matrix. Let $\mathcal{O} = (u,h,y^\star,m)$ denote a training observation, where \(y^\star=\log v\in\mathbb{R}^n\) is the observed log-demand vector, and \(m = (m_1,\ldots,m_n)\), with \(m_i \in\{0,1\}\), is an observation mask indicating which product demands are observed in the corresponding store-week\footnote{This mask accounts for a common feature of retail panels: not every SKU is observed in every store-week.}.

The training objective is
\begin{equation}
\label{eq:objective}
\mathcal{L}(\theta)
=
\mathbb{E}_{\mathcal{O}}
\Big[
\mathcal{L}_{\mathrm{fit}}(u,h,y^\star,m)
+
\lambda_{\mathrm{sm}}\mathcal{L}_{\mathrm{sm}}(u,h,m)
+
\lambda_{\mathrm{elast}}\mathcal{L}_{\mathrm{elast}}(u,h,m)
\Big].
\end{equation}
In practice, the expectation $\mathbb{E}_\mathcal{O}$ is approximated by mini-batch averages. The three terms control predictive fit, demand-surface smoothness, and elasticity plausibility, respectively.

For the fit term, let $N_m=\sum_{i=1}^n m_i$ denote the number of observed SKU demand entries in the current store-week observation. We use the robust log-demand loss
\[
\mathcal{L}_{\mathrm{fit}}(u,h,y^\star,m)
=
\frac{1}{N_m}
\sum_{i=1}^n
m_i\,
\mathrm{Huber}_{\delta}
\left(
g_{\theta,i}(u,h)-y_i^\star
\right).
\]
The Huber loss reduces sensitivity to atypical observations while preserving sensitivity to small residuals. Normalization by \(N_m\) prevents store-weeks with more observed products from receiving mechanically larger weight.

Because elasticities are derivatives of \(g_\theta\), predictive fit alone does not ensure stable price responses: an irregular fitted surface may produce noisy or economically implausible derivatives.
To control the geometry of the demand surface, we regularize the own-price curvature
\[
\kappa_i(u,h)
:=
\frac{\partial^2 g_{\theta,i}(u,h)}{\partial u_i^2}\,.
\]
Under Eq.~\ref{eq:icdn-sparse-g}, the baseline and linear terms do not contribute to this second derivative, yielding
\[
\kappa_i(u,h)
=
w_{ii}(h)^\top B_i''(u_i)
+
\sum_{j:(i,j)\in\mathcal P_\theta(x)}
a_{ij}(h_{\text{pair}})
B_i''(u_i)^\top U^{(ij)}(h_{\text{pair}})B_j(u_j).
\]
We define the smoothness penalty as
\[
\mathcal{L}_{\mathrm{sm}}(u,h,m)
=
\frac{1}{N_m}
\sum_{i=1}^n
m_i\,\kappa_i(u,h)^2.
\]
This penalty discourages excessive own-price curvature and favors smoother price responses unless sharper nonlinearities are supported by the data. Penalizing the full Hessian is possible but substantially more expensive in high-dimensional multiproduct settings.

We also regularize the implied elasticities directly, rather than their individual parameters. Let \(\ell_{ij}<r_{ij}\) denote the corresponding plausibility range. For own-price elasticities we use
\([\ell_{\mathrm{own}},r_{\mathrm{own}}]=[-5,0]\).
The upper bound reflects the empirical prior of non-positive own-price responses, while the lower bound limits extreme values and is informed by the disaggregated alcoholic-beverage estimates of \cite{Srivastava2014}.
For cross-price elasticities we use \([\ell_{\mathrm{cross}},r_{\mathrm{cross}}]=[-1,1]\), allowing both substitution and complementarity while discouraging unusually large
responses. This choice is motivated by the weak cross-price effects reported for Dominick's beer data by \cite{ToroGonzalezMcCluskeyMittelhammer2014}.
Thus,
\[
\ell_{ij}
=
\begin{cases}
\ell_{\mathrm{own}}, & i=j,\\
\ell_{\mathrm{cross}}, & i\neq j,
\end{cases}
\qquad
r_{ij}
=
\begin{cases}
r_{\mathrm{own}}, & i=j,\\
r_{\mathrm{cross}}, & i\neq j.
\end{cases}
\]
The penalty is evaluated only on observed, model-relevant entries:
\[
I_E(u,h,m)
=
\{(i,i):m_i=1\}
\cup
\{(i,j):(i,j)\in P_\theta(h),\,i\neq j,\,m_i=m_j=1\},
\]
with \(N_E=|I_E(u,h,m)|\).
The elasticity-band penalty is
\begin{equation}\label{eq:loss-elasticity}
\mathcal{L}_{\mathrm{elast}}(u,h,m)
=
\frac{1}{N_E}
\sum_{(i,j)\in\mathcal I_E}
\left[
\operatorname{ReLU}
\left(
\widehat E_{\theta,ij}(u,h)-r_{ij}
\right)^2
+
\operatorname{ReLU}
\left(
\ell_{ij}-\widehat E_{\theta,ij}(u,h)
\right)^2
\right].
\end{equation}
This is a soft constraint: elasticities may leave the specified ranges, and cross-price effects may have either sign when supported by the data.

It is important to distinguish guaranteed properties from those only encouraged during training. Positivity and row-wise integrability hold by construction under the stated conditions. By contrast, negative and moderate own-price responses, limited cross-price magnitudes, and low own-price curvature are only encouraged by the sign transformation and regularization terms.
In particular, global own-price monotonicity over a relevant domain $\mathcal{D}$ would require
\begin{equation}
\widehat{E}_{\theta,ii}(u,x)
=
\frac{\partial g_{\theta,i}(u,x)}{\partial u_i}
\leq 0,
\qquad
\forall (u,x)\in\mathcal{D},\quad i=1,\ldots,n.
\end{equation}
However, the transformation in Eq.~\ref{eq:beta_ii} constrains only the linear own-price coefficient, not the complete derivative containing spline and interaction terms.
Similarly, concavity of each fitted log-demand component with respect to the complete log-price vector would require
\begin{equation}
\nabla_u^2 g_{\theta,i}(u,x) \preceq 0,
\qquad
\forall (u,x)\in\mathcal{D},\quad i=1,\ldots,n,
\end{equation}
where \(\preceq 0\) denotes negative semidefiniteness. The smoothness penalty is sign-neutral and does not constrain the full Hessian. Consequently, ICDN does not guarantee global monotonicity or concavity, nor does it impose adding-up, homogeneity, or Slutsky symmetry.


To conclude this section, the individual building blocks of ICDN are established techniques: multilayer perceptrons \citep{goodfellow2016deep}, cubic splines \citep{deboor2001practical}, scaled dot-product attention \citep{NIPS2017_Attention}, Huber loss \citep{huber1964robust}, and softplus. ICDN combines these components for multiproduct demand estimation through context-conditioned spline coefficients, sparse directed interactions, and regularization of the implied curvature and elasticities. The contribution lies in this structured integration rather than in the individual techniques.

\section{Numerical Experiments}
\label{sec:experiments}

Our primary empirical evaluation uses the Dominick's Finer Foods (DFF) scanner dataset\footnote{For details on the Dominick's dataset and variable definitions, see \cite{kilts_dominicks_dataset,kilts_dominicks_manual_2018}.} for the beer category, with predictive generalization and elasticity stability compared against unregularized and regularized log-log benchmarks and a feature-matched pooled MLP. Subsequent experiments examine component ablations, architectural alternatives, computational scalability, and cross-dataset robustness.

Since the data are observational, the comparison concerns fitted demand surfaces and their local derivatives, not causal price effects. Moreover, DFF products are indexed by UPC codes; therefore, we refer to SKUs as UPCs throughout the empirical analysis.


\subsection{Preprocessing and targets}
\label{subsec:preprocessing}

The raw DFF records contain multi-item deals and heterogeneous package sizes. We first correct multi-item deals by computing the effective price per UPC as \(p_{\mathrm{upc}}=\mathrm{total\ price}/\mathrm{units\ per\ deal}\).
We then convert pack size into liters per UPC, denoted by \(\ell_{\mathrm{upc}}\), using \(1\,\mathrm{oz}\approx0.0295735\,\mathrm{L}\) and \(1\,\mathrm{gal}\approx3.78541\,\mathrm{L}\). This allows quantities and prices to be expressed on a common liter basis:
\[
q_L=\mathrm{units\ sold}\cdot\ell_{\mathrm{upc}},
\qquad
p_L=\frac{p_{\mathrm{upc}}}{\ell_{\mathrm{upc}}}.
\]
The empirical target and price input are \(y=\log q_L\) and \(u=\log p_L\).

Observations with non-positive prices or quantities are discarded because logarithms and elasticities are not well defined in those cases. The fitted surface therefore models positive observed demand rather than purchase/no-purchase, zero-sales, stockout, or assortment-availability processes. Moreover, observations flagged by \texttt{exclude\_flag} are also removed as a data-quality screen.
Finally, Tables~\ref{tab:final-dl-schema-identifiers}-\ref{tab:final-dl-schema-competitive-static} report the final dataset schema organized by feature family.


The exploratory data analysis (EDA) motivates both the sample restrictions and the covariates used by ICDN. The raw store-UPC panel is temporally irregular, and many series have limited longitudinal support or within-series price variation. Price changes are also strongly linked to promotional status, creating promo-price confounding.
For these reasons, we retain only store-UPC series with at least 52 observed weeks, within-span coverage of at least 0.75, at least 3 distinct price values, at least 5 price changes, and a within-series log-price range of at least 0.15. To reduce promo-driven confounding, we also remove series with absolute price-promotion correlation above 0.80 or with more than 80\% of price changes coinciding with promotion-state switches. 


We then remove log-price outliers within each UPC. The preprocessing pipeline reduces the sample from 1,966,147 basic-quality-filtered rows to 463,722 cleaned store-UPC-week observations. The resulting sample retains a negative (\(\approx -2\)) log-price--log-demand association, while simple log-log regressions have modest explanatory power, indicating that price explains only part of the variation in demand.


The EDA also motivates the temporal feature design. To account for temporal incompleteness, persistence, seasonality, and lifecycle effects, each store-UPC history is completed on a weekly calendar grid. Lagged, rolling, seasonal, and gap-aware covariates are computed on this grid and then projected back onto the observed rows used for estimation. All lagged and rolling variables use past information only and are recomputed within each training split to prevent validation- or test-period leakage. ICDN is therefore trained only on observed store-UPC-week rows while using temporal features that account for missing weeks.

\subsection{Training protocol}
\label{subsec:model-training}

ICDN is trained in two phases to stabilize optimization before introducing the full nonlinear specification. The first provides a log-linear warm start on smoothed demand targets, and the second starts from this checkpoint and trains the full nonlinear model on the original targets.

For each training split, the spline knots and normalization constants in Eqs.~\ref{eq:knots}-\ref{eq:normalization} are estimated using only the corresponding training fold. The model is trained and validated on store-week observations in wide multiproduct format, with the fit loss evaluated only on observed UPCs through the observation mask.

The log-price vector is completed separately from the demand observation mask. Missing prices are first forward- and backward-filled within each store; remaining values are filled with the corresponding UPC-level column mean. ICDN therefore receives a complete log-price vector \(u\in\mathbb{R}^n\), while \(m\) determines which demand outcomes enter the fit loss. Imputed prices represent price-state inputs, not evidence that a UPC was active or available in that store-week. Reported cross-price elasticities are restricted to pairs for which both UPCs are observed.


Across both phases, each mini-batch contains the contextual covariates, completed log-price vectors, and observation mask. The forward pass returns predicted log-demand and the quantities required by the regularized loss. Optimization uses AdamW, with weight decay applied to standard network weights but not to biases or the output layers generating price-response coefficients. Learning rates are reduced on validation plateaus, gradients are clipped, and the best validation checkpoint is retained. GPU training uses automatic mixed precision.

\subsubsection{Training Phases}

\paragraph{Phase 0: Smoothed log-linear warm start.}

The first phase trains a restricted version of ICDN to initialize the demand surface around a stable downward-sloping price response. The nonlinear spline heads \(w_{ij}(h)\) and spline-spline interaction matrices \(U^{(ij)}(h)\) are frozen, yielding a context-dependent log-linear specification in which baseline demand and linear price responses are learned first. The own-price coefficient is initialized using the EDA-based prior \(\beta_{\mathrm{EDA}}=-2\): the bias of the own-price head is initialized through the inverse softplus transformation in Eq.~\ref{eq:beta_ii}, while its output-layer weights are set to zero, so that \(\beta_{ii}(x)\approx\beta_{\mathrm{EDA}}\) initially without context-driven variation. Phase~0 uses eight-week rolling averages of log-demand as targets to attenuate high-frequency noise, while the Huber fit, elasticity-band, and own-price smoothness terms remain active. The resulting checkpoint provides calibrated baseline demand and economically plausible own-price slopes before nonlinear flexibility is introduced.

\paragraph{Phase 1: Full nonlinear demand model.}
The second phase starts from the best Phase~0 checkpoint and is trained on the original, non-smoothed demand targets. The nonlinear price-response heads \(w_{ij}\) and \(U^{(ij)}\) are unfrozen, allowing the model to learn own- and cross-price spline effects and nonlinear cross-price interactions, with the full ICDN loss remaining active. The retained Phase~1 checkpoint is used for validation, test-time prediction, and elasticity extraction. Elasticities are evaluated only for observed UPC entries, and cross-price elasticities are reported only when both the demand-receiving and
price-perturbing UPCs are observed in the corresponding store-week.




\subsection{Hyperparameter optimization}\label{sec:hyperparameter_optimization}

The main ICDN hyperparameters are selected with Optuna (\cite{optuna_2019}) within the nested temporal validation protocol  with five outer folds using two validation objectives: predictive fit and elasticity plausibility. For each outer training split, 20 candidate configurations are evaluated over three expanding inner temporal folds and three tuning seeds, yielding nine fold-seed evaluations per candidate. The corresponding outer-validation period is therefore never used for hyperparameter selection. An independent search is performed for the final holdout using only the pre-holdout sample. The search includes the hidden-layer configuration, dropout, batch size, spline basis size \(K\), phase-specific learning rates, and the smoothness and elasticity-band regularization strengths. Table~\ref{tab:best-hparams} reports the configuration selected independently for each outer fold and for the holdout. The complete search required more than one day on an NVIDIA RTX 5070 Ti GPU.

The first objective is predictive performance, measured by the global validation \(R^2\) on observed log-demand entries. For a validation set indexed by store-week observations \(b\) and UPCs \(i\), with observation mask \(m_{bi}\in\{0,1\}\), we compute
\[
R^2
=
1
-
\frac{
\sum_{b,i} m_{bi}\left(y_{bi}-\widehat y_{bi}\right)^2
}{
\sum_{b,i} m_{bi}\left(y_{bi}-\bar y_m\right)^2
},
\qquad
\bar y_m
=
\frac{\sum_{b,i}m_{bi}y_{bi}}{\sum_{b,i}m_{bi}}.
\]
Here \(y_{bi}\) and \(\widehat y_{bi}\) are the observed and predicted log-demands for UPC \(i\) in store-week \(b\), respectively, and \(\bar y_m\) is the masked validation mean computed over observed entries.

The second objective is an elasticity score, designed to summarize the economic plausibility of the model-implied price responses on validation data. This score is used only for hyperparameter selection; it is not an additional training loss. Let \(\mathcal E_{\mathrm{own}}\) and \(\mathcal E_{\mathrm{cross}}\) denote the set of own- and cross-price elasticities evaluated on observed validation entries. 
The own-price score combines an in-range share with a distributional alignment term. Using the same reference interval as in Eq.~\ref{eq:loss-elasticity},
\[
P_{\mathrm{own}}
=
\frac{1}{|\mathcal E_{\mathrm{own}}|}
\sum_{e\in\mathcal E_{\mathrm{own}}}
\mathbf{1}\{-5\leq e\leq 0\}.
\]
Because \(P_{\mathrm{own}}\) does not control the center of the distribution, we also compare its median \(\widetilde e_{\mathrm{own}} =\operatorname{median}\{e:e\in \mathcal E_{\mathrm{own}}\}\) with the EDA reference \(\beta_{\mathrm{EDA}}=-2\):
\[
P_{\mathrm{prior}}
=
\min\left\{
\frac{
\max\left(0,
|\widetilde e_{\mathrm{own}}-\beta_{\mathrm{EDA}}|-0.3
\right)}
{|\beta_{\mathrm{EDA}}|},
1
\right\}.
\]
The tolerance of \(0.3\) permits moderate deviations from the reference before the penalty becomes active.
The resulting own-price score is
\[
S_{\mathrm{own}}
=
P_{\mathrm{own}}(1-P_{\mathrm{prior}}).
\]

For cross-price elasticities, using the same reference interval as in Eq.~\ref{eq:loss-elasticity},
\[
S_{\mathrm{cross}}
=
\frac{1}{|E_{\mathrm{cross}}|}
\sum_{e\in E_{\mathrm{cross}}}
\mathbf{1}\{-1\leq e\leq1\},
\]
with \(S_{\mathrm{cross}}=1\) when no cross-price elasticity is available. The combined elasticity score is then
\[
S_{\mathrm{elast}}
=
0.7S_{\mathrm{own}}+0.3S_{\mathrm{cross}},
\]
placing greater weight on the more reliably identified own-price responses.


For trial \(t\), let \(\mathcal H_t\) denote its nine inner fold-seed evaluations. We penalize variability across these evaluations through
\[
R_{\mathrm{robust}}^{2,(t)}
=
\overline{R^2}^{(t)}
-
0.25\,\mathrm{sd}_{\sigma\in\mathcal H_t}\left(R_\sigma^2\right), \quad \text{and} \quad S_{\mathrm{elast,robust}}^{(t)}
=
\overline{S}_{\mathrm{elast}}^{(t)}
-
0.25\,\mathrm{sd}_{\sigma\in\mathcal H_t}\left(S_{\mathrm{elast},\sigma}\right),
\]
where the bars denote averages over the nine inner fold-seed evaluations \(\sigma\in\mathcal H_t\). The final configuration is selected using
\[
S_{\mathrm{select}}^{(t)}
=
R_{\mathrm{robust}}^{2,(t)}
+
S_{\mathrm{elast,robust}}^{(t)},
\qquad
t^\star
=
\arg\max_{t\in T}S_{\mathrm{select}}^{(t)}, 
\]
where \(\mathcal T\) denotes the set of completed candidate trials considered for final selection. Table~\ref{tab:best-hparams} reports the configurations selected independently within each outer training split and within the pre-holdout sample.

\begin{table}[h!]
\centering
\caption{ICDN hyperparameters selected independently by the nested temporal Optuna search for each outer fold and for the final holdout. Reported performance corresponds to the inner temporal folds used for model selection.}
\label{tab:best-hparams}
\scriptsize
\setlength{\tabcolsep}{3.2pt}
\renewcommand{\arraystretch}{1.1}
\begin{tabular}{lcccccc}
\toprule
& \textbf{Outer 1} & \textbf{Outer 2} & \textbf{Outer 3}
& \textbf{Outer 4} & \textbf{Outer 5} & \textbf{Holdout} \\
\midrule
Selected trial
& 3 & 1 & 9 & 1 & 11 & 17 \\

Selection score
& 1.4602 & 1.3569 & 1.3925 & 1.3358 & 1.2161 & 1.2369 \\

Mean validation \(R^2\)
& 0.5444 & 0.5377 & 0.5197 & 0.5241 & 0.4978 & 0.5068 \\

Std. validation \(R^2\)
& 0.0352 & 0.0488 & 0.0320 & 0.0105 & 0.0105 & 0.0262 \\

Mean elasticity score
& 0.9336 & 0.8568 & 0.8973 & 0.8327 & 0.7558 & 0.7574 \\

Std. elasticity score
& 0.0361 & 0.1015 & 0.0659 & 0.0736 & 0.1395 & 0.0833 \\

Batch size
& 512 & 256 & 1024 & 1024 & 256 & 512 \\

Dropout
& 0.0082 & 0.2875 & 0.0059 & 0.0014 & 0.0201 & 0.2361 \\

Hidden layers
& \((128,64)\)
& \((64,32)\)
& \((256,128)\)
& \((128,64)\)
& \((256,128)\)
& \((128,64)\) \\

Spline basis size \(K\)
& 8 & 12 & 14 & 16 & 7 & 11 \\

\(\mathrm{LR}_{P0}\)
& \(5.212{\times}10^{-4}\)
& \(4.706{\times}10^{-4}\)
& \(4.171{\times}10^{-4}\)
& \(9.153{\times}10^{-4}\)
& \(4.467{\times}10^{-4}\)
& \(8.343{\times}10^{-3}\) \\

\(\mathrm{LR}_{P1}\)
& \(7.732{\times}10^{-5}\)
& \(2.821{\times}10^{-3}\)
& \(3.533{\times}10^{-5}\)
& \(3.099{\times}10^{-4}\)
& \(2.867{\times}10^{-3}\)
& \(1.960{\times}10^{-3}\) \\

\(\lambda_{\mathrm{sm}}\)
& \(1.185{\times}10^{-1}\)
& \(3.059{\times}10^{-3}\)
& \(1.294{\times}10^{-2}\)
& \(9.935{\times}10^{-4}\)
& \(2.301{\times}10^{-3}\)
& \(1.429{\times}10^{-4}\) \\

\(\lambda_{\mathrm{elast}}\)
& \(1.978{\times}10^{-2}\)
& \(5.006{\times}10^{-2}\)
& \(9.825{\times}10^{-5}\)
& \(1.659{\times}10^{-2}\)
& \(1.520{\times}10^{-5}\)
& \(6.231{\times}10^{-2}\) \\
\bottomrule
\end{tabular}
\end{table}

\subsection{Temporal validation and resampling}

Within each outer training split, hyperparameters are selected using only the corresponding inner temporal folds. The selected configuration is then refitted on that outer training sample and evaluated on the subsequent outer-validation period using five random seeds. Repeating this procedure across five expanding outer folds separates hyperparameter selection from final assessment, while the independent final holdout uses its own configuration selected exclusively from the pre-holdout sample. Temporal folds evaluate generalization to later periods, whereas random seeds assess sensitivity to stochastic training elements.


We additionally use block bootstrap resampling over training weeks to assess elasticity stability. Weeks are grouped into non-overlapping temporal blocks and sampled with replacement, preserving short-run dependence within each block. Re-estimation across bootstrap samples yields elasticity distributions from which standard deviations and confidence intervals are computed. These bootstrap quantities are used only as uncertainty and stability diagnostics, not to define the demand model.

\subsection{Evaluation benchmarks}
\label{subsec:evaluation}

The evaluation includes three complementary benchmarks. OLS and Ridge are estimated on the same directed pairwise dataset and compact control set, whereas the generic demand-first MLP uses the same wide store-week representation and contextual covariates as ICDN and obtains elasticities through automatic differentiation. Together, these benchmarks distinguish the effects of coefficient shrinkage and generic nonlinear information pooling from ICDN's structured price representation and elasticity-oriented regularization. These models provide progressively stronger reference points rather than an exhaustive comparison with all possible predictive architectures.

For OLS and Ridge, the processed store-UPC-week panel is transformed into a directed pairwise dataset. For each store-week and ordered pair of distinct UPCs \((i,j)\), one row contains the demand and own price of UPC \(i\) together with the price of UPC \(j\). Thus, \(i\) is the demand-receiving product and \(j\) the price-perturbing product; \((i,j)\) and \((j,i)\) are treated as distinct observations.


For each store \(s\) and directed UPC pair \((i,j)\), the OLS and Ridge benchmarks use the common log-log specification
\begin{equation}
\log v_{sti}
=
\beta_{0,sij}
+
\beta^{\mathrm{own}}_{sij}\log p_{sti}
+
\beta^{\mathrm{cross}}_{sij}\log p_{stj}
+
z_{sti}^{\top}\gamma_{sij}
+
\varepsilon_{sti}.
\label{eq:benchmark-pairwise}
\end{equation}
Here \(v_{sti}\) denotes liters sold for UPC \(i\) in store \(s\) and week \(t\), \(p_{sti}\) is the price per liter of UPC \(i\), and \(p_{stj}\) is the price per liter of candidate UPC \(j\) in the same store-week. The term \(\beta_{0,sij}\) is an intercept specific to store \(s\) and directed pair \((i,j)\), capturing the baseline log-demand level in that local regression. The vector \(z_{sti}\) contains the benchmark controls listed in Table~\ref{tab:benchmark-controls}. The disturbance \(\varepsilon_{sti}\) captures residual variation in log-demand that is not explained by the two prices and the included controls, and, finally, the coefficients \(\beta^{\mathrm{own}}_{sij}\) and \(\beta^{\mathrm{cross}}_{sij}\) are the corresponding own- and directed cross-price elasticities and are constant within each store-pair regression.


To limit redundancy and collinearity in the local OLS and Ridge regressions, \(z_{sti}\) uses the compact control subset reported in Table~\ref{tab:benchmark-controls} rather than the complete deep-learning feature set. Table~\ref{tab:benchmark-model-summary} summarizes the main differences among the evaluation models.

\begin{table}[h!]
\centering
\small
\setlength{\tabcolsep}{4pt}
\renewcommand{\arraystretch}{1.15}
\caption{Control variables used in the evaluation benchmarks. The benchmarks use a compact, collinearity-aware subset of the final feature set rather than all available covariates.}
\label{tab:benchmark-controls}
\begin{tabularx}{\textwidth}{@{}p{0.43\textwidth}X@{}}
\toprule
\textbf{Variable(s)}
&
\textbf{Definition and motivation}
\\
\midrule

\path|on_promo|
&
Promotion indicator. Controls for the promotional state of the demand-receiving product, which is strongly correlated with its price.
\\

\path|week_rank|
&
Gap-free sequential time index. Captures smooth aggregate time trends in the weekly panel.
\\

\begin{minipage}[t]{\linewidth}
\path|sin_52|\\
\path|cos_52|
\end{minipage}
&
Annual Fourier seasonality terms. Capture yearly demand patterns.
\\

\begin{minipage}[t]{\linewidth}
\path|sin_13|\\
\path|cos_13|
\end{minipage}
&
Quarterly Fourier seasonality terms. Capture shorter recurrent seasonal patterns.
\\

\path|weeks_since_first_seen_store_upc|
&
Weeks since the UPC first appears in the corresponding store. Controls for local product lifecycle effects such as rollout and stabilization.
\\

\begin{minipage}[t]{\linewidth}
\path|lag_1_log_liters_sold|\\
\path|lag_4_log_liters_sold|
\end{minipage}
&
Lagged log demand for the same store-UPC series. Captures short-run and monthly-like demand persistence.
\\

\begin{minipage}[t]{\linewidth}
\path|miss_lag_1|\\
\path|miss_lag_4|
\end{minipage}
&
Missingness indicators for the retained lagged demand features.
\\

\path|promo_intensity_store_week|
&
Share of UPCs on promotion in the same store-week. Controls for the store-wide promotional environment faced by the focal UPC.
\\

\path|n_neighbors_sw_cat|
&
Number of other UPCs in the same store-week-category. Measures the size of the active category-level competitive set.
\\

\path|neighbor_promo_share_sw_cat|
&
Share of category neighbors on promotion. Captures promotional pressure from rival UPCs in the same category.
\\

\path|lag1_neighbor_mean_log_liters_sold|
&
Mean lag-1 log demand across valid category neighbors. Summarizes recent local category-demand conditions.
\\

\path|share_new_neighbors_13w|
&
Share of neighbors that are new in-store within the previous 13 weeks. Controls for changes in the competitive set due to product introductions or activations.
\\

\bottomrule
\end{tabularx}
\end{table}


\begin{table}[h!]
\centering
\small
\setlength{\tabcolsep}{3.5pt}
\renewcommand{\arraystretch}{1.15}
\caption{Summary of the evaluation models.}
\label{tab:benchmark-model-summary}
\begin{tabularx}{\textwidth}{@{}
p{0.05\textwidth}
p{0.18\textwidth}
p{0.18\textwidth}
p{0.17\textwidth}
p{0.17\textwidth}
X@{}}
\toprule
\textbf{Model}
&
\textbf{Estimation unit}
&
\textbf{Price response}
&
\textbf{Elasticity}
&
\textbf{Pooling}
&
\textbf{Regularization}
\\
\midrule

OLS
&
Store-pair regression
&
Linear log-log
&
Regression coefficients
&
None
&
No coefficient penalty; HC1 robust inference
\\

Ridge
&
Store-pair regression
&
Linear log-log
&
Regression coefficients
&
None
&
\(L_2\) coefficient shrinkage selected by internal cross-validation
\\

MLP
&
Single global multiproduct model
&
Generic nonlinear function
&
Automatic differentiation
&
Global pooling with store, brand, and style embeddings
&
Weight decay and early stopping; no elasticity-specific constraints
\\

ICDN
&
Single structured multiproduct model
&
Linear and spline-based nonlinear price surface
&
Analytic derivatives
&
Shared contextual encoder and sparse directed interactions
&
Curvature penalty, elasticity-band penalty, and constrained own-price
linear component
\\

\bottomrule
\end{tabularx}
\end{table}

\paragraph{Unpenalized OLS benchmark.}
Equation~\ref{eq:benchmark-pairwise} is estimated separately for each store and directed UPC pair, without pooling across stores or products. HC1 heteroskedasticity-consistent standard errors are used, and we retain the price coefficients, confidence intervals, \(p\)-values, and validation-period predictions. This benchmark provides direct econometric interpretation but may be sensitive to small local samples, correlated regressors, and limited within-group price variation.



\paragraph{Ridge-regularized benchmark.}
Ridge (by \cite{sklearnRidgeCV}) preserves the directed dataset, controls, temporal split, grouping structure, and absence of pooling used by OLS, changing only the estimator. For each store \(s\) and directed pair \((i,j)\), it solves 
\begin{equation}
\left(
\widehat{\beta}_{0,sij}^{(\alpha)},
\widehat{\theta}_{sij}^{(\alpha)}
\right)
=
\arg\min_{\beta_0,\theta}
\left\{
\sum_{t\in\mathcal{T}^{\mathrm{train}}_{sij}}
\left(
\log v_{sti}
-
\beta_0
-
x_{stij}^{\top}\theta
\right)^2
+
\alpha\lVert\theta\rVert_2^2
\right\},
\label{eq:benchmark-pairwise-ridge}
\end{equation}
where
\begin{equation}
\theta_{sij}
=
\left(
\beta^{\mathrm{own}}_{sij},
\beta^{\mathrm{cross}}_{sij},
\gamma_{sij}^{\top}
\right)^{\top}, \qquad \text{and} \qquad x_{stij}
=
\left(
\log p_{sti},
\log p_{stj},
z_{sti}^{\top}
\right)^{\top}.
\end{equation}
The intercept is unpenalized, while the own-price, cross-price, and control coefficients receive \(L_2\) shrinkage.

Non-intercept regressors are standardized using training-group means and standard deviations, which are then applied to validation data. The regularization parameter is selected independently for each store-pair regression by five-fold internal cross-validation over $\alpha\in\{0.1,1,10,30,100,300\}$. This internal cross-validation is performed entirely within the corresponding outer-training sample; no outer-validation or holdout observations are used to select $\alpha$. Coefficients are subsequently transformed back to their original scales so that the price coefficients retain their elasticity interpretation.


Ridge remains local and unpooled, isolating the effect of coefficient shrinkage from the global pooling and structured architecture used by the neural models. Unlike OLS, its evaluation does not use HC1 intervals or coefficient \(p\)-values and instead focuses on prediction, elasticity magnitudes, and resampling and temporal stability.


\paragraph{Generic demand-first MLP benchmark.}
The MLP is a single global multiproduct model trained on the same wide store-week representation as ICDN. It receives the complete log-price vector \(u\) and exactly the same contextual covariates \(x\) as ICDN (Tables~\ref{tab:final-dl-schema-identifiers}--\ref{tab:final-dl-schema-competitive-static}).  Store, brand, and style are represented through learned embeddings, while UPC identity is preserved by its fixed position in the multiproduct representation. The model jointly predicts the \(n\) log-demands as
\begin{equation}
\widehat{y}^{\mathrm{MLP}}
=
f_{\phi}(u,x)
\in \mathbb{R}^{n}.
\label{eq:benchmark-mlp}
\end{equation}
The network uses fully connected layers with GELU activations. Hidden-layer dimensions, dropout, learning rate, and batch size are selected under the same nested temporal Optuna protocol used for ICDN. For each outer training split and for the final pre-holdout search, the MLP receives the same budget of 20 candidate configurations, evaluated over the same three expanding inner temporal folds and three tuning seeds per candidate. The configurations selected independently for each outer fold and for the final holdout are reported in Table~\ref{tab:best-hparams-mlp}. Training uses AdamW, weight decay, and validation-based early stopping. The fit term is the same masked Huber loss \(\mathcal{L}_{\mathrm{fit}}\) defined in Section~\ref{sec:training-objective}, using the same observation mask \(m\), but without the curvature and elasticity-band penalties \(\mathcal{L}_{\mathrm{sm}}\) and \(\mathcal{L}_{\mathrm{elast}}\).

The complete own- and cross-price elasticity matrix is obtained by automatic differentiation of the fitted demand surface: $\widehat{E}^{\mathrm{MLP}}_{ij}(u,x) = \partial f_{\phi,i}(u,x) /
\partial u_j$, with $i,j=1,\ldots,n$.  Unlike ICDN, the MLP does not impose an explicit linear-spline price decomposition, sparse directed interactions, analytic price derivatives, or derivative-specific regularization. It therefore provides a feature-matched generic neural benchmark for evaluating the structured price-response architecture of ICDN.

\begin{table}[h!]
\centering
\caption{MLP hyperparameters selected independently by the nested temporal
Optuna search for each outer fold and for the final holdout. Reported
performance corresponds to the inner temporal folds used for model
selection.}
\label{tab:best-hparams-mlp}
\scriptsize
\setlength{\tabcolsep}{3.2pt}
\renewcommand{\arraystretch}{1.1}
\begin{tabular}{lcccccc}
\toprule
& \textbf{Outer 1} & \textbf{Outer 2} & \textbf{Outer 3}
& \textbf{Outer 4} & \textbf{Outer 5} & \textbf{Holdout} \\
\midrule

Selected trial
& 1 & 12 & 14 & 10 & 3 & 14 \\

Selection score
& 0.9872 & 0.9658 & 0.9771 & 0.9483 & 0.9713 & 0.9510 \\

Mean validation \(R^2\)
& 0.4813 & 0.4799 & 0.4900 & 0.4631 & 0.4735 & 0.4626 \\

Std. validation \(R^2\)
& 0.0449 & 0.0726 & 0.0184 & 0.0312 & 0.0198 & 0.0319 \\

Mean elasticity score
& 0.5206 & 0.5101 & 0.4969 & 0.4976 & 0.5075 & 0.5012 \\

Std. elasticity score
& 0.0137 & 0.0243 & 0.0208 & 0.0185 & 0.0187 & 0.0190 \\

Batch size
& 512 & 512 & 256 & 512 & 512 & 256 \\

Dropout
& 0.2939 & 0.1948 & 0.2144 & 0.2386 & 0.2564 & 0.0591 \\

Hidden layers
& \((64,32)\)
& \((64,32)\)
& \((128,64)\)
& \((128,64)\)
& \((64,32)\)
& \((128,64)\) \\

Learning rate
& \(1.883{\times}10^{-3}\)
& \(2.617{\times}10^{-3}\)
& \(4.760{\times}10^{-4}\)
& \(1.000{\times}10^{-3}\)
& \(1.936{\times}10^{-3}\)
& \(4.931{\times}10^{-3}\) \\

\bottomrule
\end{tabular}
\end{table}


\subsection{Results}
\label{sec:results}

We compare ICDN with the OLS, Ridge, and MLP benchmarks in terms of out-of-sample demand prediction and the bootstrap and temporal stability of model-implied elasticities.
For comparability, the evaluation is restricted to four UPCs:
\[
\{3410010505,\,1820000784,\,7289000011,\,8248812345\}.
\]
The subset is selected to obtain a dense and temporally balanced multiproduct panel. Specifically, common store-week support is evaluated separately across the five outer-validation periods, and candidate product sets are prioritized according to the minimum support observed across these periods. This penalizes UPCs that are dense in aggregate but become sparse or disappear during part of the study horizon. The subset is fixed before model fitting and is selected using product-availability and support information only; model predictions, predictive errors, and estimated elasticities are not used in its construction. With \(n=4\), the effective neighbor count is \(k_{\mathrm{eff}}=\min\{k,n-1\}=3=n-1\), so all possible cross-product neighbors are retained and the interaction graph is effectively dense.

Direct predictive comparisons are computed on matched validation cells across models. OLS and Ridge pairwise predictions are averaged over partner products \(j\) before obtaining one prediction per \((\mathrm{store},\mathrm{UPC}_i,\mathrm{week},\mathrm{fold})\) observation.
For elasticity comparisons, pairwise own-price estimates are likewise averaged over partner products within each bootstrap run, while ICDN and MLP local elasticities are averaged over the corresponding observed market instances. Bootstrap standard deviations and percentile intervals are then computed from these run-specific averages.
Because the data are observational, these comparisons assess predictive performance and derivative reproducibility rather than recovery of causal price effects.

\paragraph{Price variation and MLP elasticity identification.}
Figure~\ref{fig:price-isolation-mlp} shows that the MLP recovers a substantially weaker median own-price elasticity than the estimators with an explicit price channel. This result should be interpreted in light of its price-entangled representation: unlike OLS, Ridge, and ICDN, the MLP does not structurally isolate price from the remaining contextual covariates. Consequently, when independent price variation is limited, the network may explain demand variation through correlated inputs rather than through the price coordinate itself, yielding a weak local price derivative.

This pattern is also visible at the UPC level. UPC~7289000011 exhibits the smallest price variation (\(\mathrm{sd}(\log p)=0.059\)) and a substantial correlation between price and promotion (\(-0.479\)), while producing the MLP own-price elasticity closest to zero (\(-0.064\)). The remaining UPCs exhibit greater price variation and materially more negative MLP elasticities (\(-0.288\) to \(-0.409\)). 

\begin{figure}[h!]
    \centering
    \includegraphics[width=0.5\textwidth]{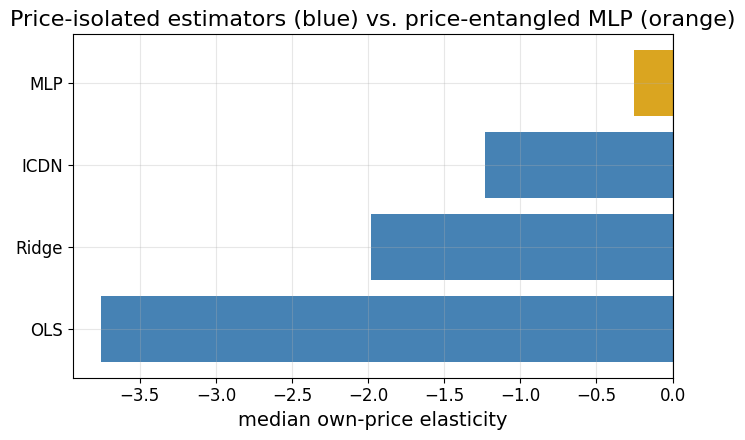}
    \caption{Median own-price elasticity across benchmark models. OLS, Ridge, and ICDN estimate price effects through an explicit price channel, whereas the MLP learns price jointly with the remaining contextual covariates.}
    \label{fig:price-isolation-mlp}
\end{figure}

\paragraph{Predictive generalization.}
Table~\ref{tab:generalization_results} reports the out-of-sample predictive comparison. Fold-level MAE, RMSE, and \(R^2\) are calculated globally after pooling the same set of store-UPC-week observations for all models. We define
\[
\Delta R^2
=
R^2_{\mathrm{ICDN}}
-
R^2_{\mathrm{benchmark}},
\tag{25}
\]
so positive values favor ICDN. With only five outer temporal folds, the paired \(t\)-tests are descriptive diagnostics rather than strong inferential evidence.
\begin{table}[h!]
\centering
\small
\setlength{\tabcolsep}{11pt}
\renewcommand{\arraystretch}{1.15}
\caption{Out-of-sample predictive comparison. Positive \(\Delta R^2\) values favor ICDN. Paired \(t\)-tests are descriptive diagnostics based on five temporal folds.}
\label{tab:generalization_results}
\begin{tabular}{lccc}
\toprule
\textbf{Metric} & \textbf{OLS} & \textbf{Ridge} & \textbf{MLP} \\
\midrule
Comparable temporal folds & 5 & 5 & 5 \\
Mean \(\Delta R^2\) & \(0.3185\) & \(0.0648\) & \(0.0260\) \\
Paired \(t\)-statistic & \(1.946\) & \(3.040\) & \(2.129\) \\
Paired \(t\)-test \(p\)-value & \(0.1236\) & \(0.0384\) & \(0.1003\) \\
\midrule
Interpretation
& Favors ICDN
& Small ICDN difference
& Small ICDN difference \\
\bottomrule
\end{tabular}
\end{table}

ICDN achieves positive mean \(\Delta R^2\) against all three benchmarks: \(0.3185\) relative to OLS, \(0.0648\) relative to Ridge, and \(0.0260\) relative to the MLP.  Thus, the largest average predictive gain is observed against OLS, while the differences relative to Ridge and especially the MLP are smaller.


Figure~\ref{fig:generalization_results} provides the corresponding matched store-UPC-fold comparison. ICDN attains lower MAE in more than half of the matched cells against each benchmark, with the MLP remaining the closest neural predictive competitor.
\begin{figure}[h!]
    \centering
    \includegraphics[width=\textwidth]
    {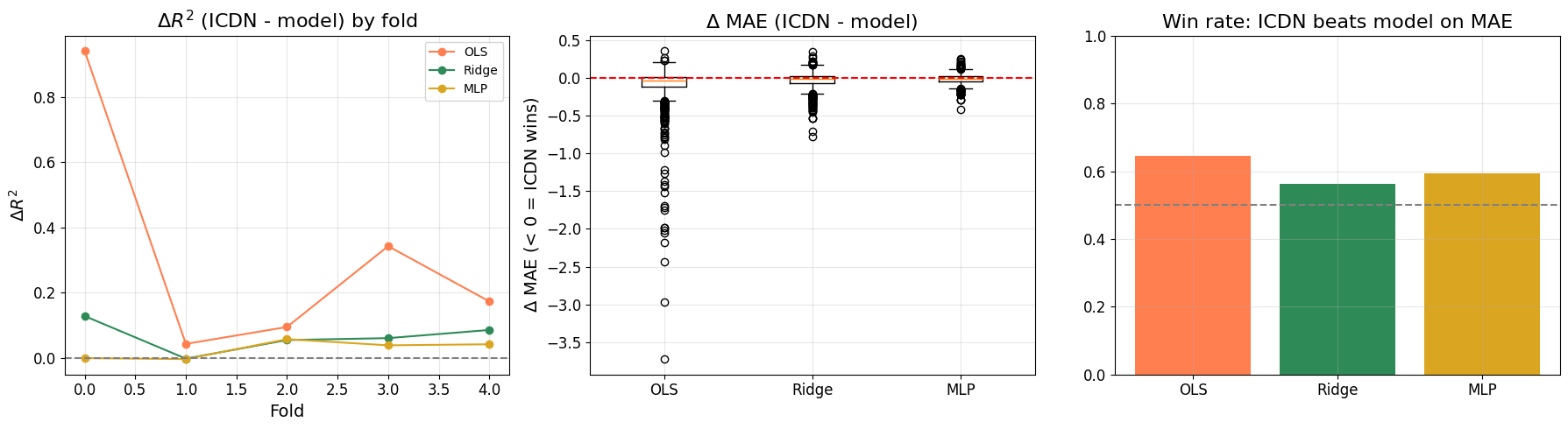}
    \caption{Out-of-sample predictive comparison. Left: fold-level \(\Delta R^2\) (positive favors ICDN). Middle: matched store-UPC-fold \(\Delta\mathrm{MAE}\)(negative favors ICDN). Right: proportion of matched cells with lower ICDN MAE; the reference line marks \(50\%\).}
    \label{fig:generalization_results}
\end{figure}

\paragraph{Own-price elasticity stability.}
Table~\ref{tab:own_stability_results} reports two complementary own-price elasticity dispersion diagnostics over 212 matched store-UPC series. Bootstrap dispersion summarizes sensitivity to resampled training sets, whereas inter-fold standard deviations summarize variation in mean elasticities across the outer temporal validation folds. Because the latter also reflect changes in the estimation window and evaluation context, they should be interpreted as temporal variation rather than as pure estimator reproducibility.

\begin{table}[h!]
\centering
\small
\setlength{\tabcolsep}{11pt}
\renewcommand{\arraystretch}{1.15}
\caption{Absolute dispersion and temporal variation of own-price elasticities over 212 matched store-UPC series. Comparative rates report the proportion of matched series in which ICDN produces the narrower or less variable estimate.}
\label{tab:own_stability_results}
\begin{tabular}{lccc}
\toprule
\textbf{Metric} & \textbf{OLS} & \textbf{Ridge} & \textbf{MLP} \\
\midrule
ICDN narrower interval rate & 94.8\% & 95.8\% & 0.0\% \\
ICDN lower bootstrap std rate & 95.8\% & 96.2\% & 0.0\% \\
Same-sign rate & 99.5\% & 99.5\% & 100.0\% \\
Median interval width, ICDN & 0.794 & 0.794 & 0.794 \\
Median interval width, benchmark & 2.650 & 2.543 & 0.283 \\
\midrule
ICDN lower inter-fold std rate & 73.1\% & 76.4\% & 0.0\% \\
Median inter-fold std, ICDN & 0.381 & 0.381 & 0.381 \\
Median inter-fold std, benchmark & 0.641 & 0.606 & 0.067 \\
\bottomrule
\end{tabular}
\end{table}

ICDN produces substantially more concentrated own-price bootstrap distributions than OLS and Ridge. Its percentile interval is narrower in \(94.8\%\) and \(95.8\%\) of matched series, respectively, and the corresponding rates based on bootstrap standard deviations are \(95.8\%\) and \(96.2\%\). The median ICDN interval width is \(0.794\), compared with \(2.650\) for OLS and \(2.543\) for Ridge. Despite the coefficient shrinkage imposed by Ridge, its own-price elasticity estimates therefore remain substantially more dispersed under bootstrap resampling than those of ICDN.

The comparison with the MLP requires a different interpretation. The MLP has narrower bootstrap intervals and lower bootstrap standard deviations in all matched series, with a median interval width of only \(0.283\). However, as shown in the price-identification analysis above, its own-price elasticities are also strongly concentrated near zero. Absolute dispersion is scale-dependent, so an elasticity distribution centered around a very small effect can exhibit a small standard deviation without necessarily providing a more informative estimate of price responsiveness. 

The temporal results show a similar distinction. ICDN has lower inter-fold variation than OLS in \(73.1\%\) of matched series and than Ridge in \(76.4\%\), with a median inter-fold standard deviation of \(0.381\), compared with \(0.641\) and \(0.606\), respectively. The MLP again exhibits substantially smaller absolute variation (\(0.067\)). 

To account for the large difference in elasticity magnitude, we complement these absolute measures with the signal-to-noise ratio
\[
\mathrm{SNR}
=
\frac{|\overline{E}|}{\mathrm{sd}(E)},
\]
where \(\overline{E}\) and \(\mathrm{sd}(E)\) denote the bootstrap mean and standard deviation for each store-UPC elasticity series. 

This comparison changes the interpretation of the MLP result. ICDN has a higher bootstrap SNR in \(82.5\%\) of the matched series, with a median SNR of \(3.208\) compared with \(2.701\) for the MLP. Thus, although the MLP exhibits substantially lower absolute bootstrap dispersion, this advantage is partly associated with the much smaller magnitude of its recovered price response. Relative to the estimated effect size, ICDN exhibits lower dispersion than the MLP in most matched comparisons\footnote{The SNR is therefore used here as a complementary diagnostic of the MLP comparison rather than as an overall ranking of stability across models.}. 

\begin{table}[h!]
\centering
\small
\setlength{\tabcolsep}{10pt}
\renewcommand{\arraystretch}{1.15}
\caption{Relative bootstrap dispersion of own-price elasticities.}
\label{tab:own_stability_snr}
\begin{tabular}{lc}
\toprule
\textbf{Metric}  & \textbf{MLP} \\
\midrule
ICDN higher SNR rate  & 82.5\% \\
Median SNR, ICDN  & 3.208 \\
Median SNR, benchmark  & 2.701 \\
Median \(|E|\), ICDN  & 0.615 \\
Median \(|E|\), benchmark  & 0.211 \\
\bottomrule
\end{tabular}
\end{table}

Directional agreement is nearly complete, with same-sign rates of \(99.5\%\) relative to OLS and Ridge and \(100\%\) relative to the MLP. This agreement should nevertheless be interpreted in light of ICDN's design, since its own-price parameterization and elasticity regularization explicitly encourage negative and economically moderate price responses. Section~\ref{sec:ablation} examines this issue separately.

\paragraph{Cross-price elasticity stability.}
Table~\ref{tab:cross_stability_results} reports the corresponding diagnostics for 544 matched directed store-UPC cross-pairs. As cross-price effects are weaker and more specification-sensitive than own-price responses, bootstrap dispersion and inter-fold variation are reported separately. 

\begin{table}[h!]
\centering
\small
\setlength{\tabcolsep}{10pt}
\renewcommand{\arraystretch}{1.15}
\caption{Cross-price elasticity dispersion and temporal-variation diagnostics. Bootstrap and temporal comparisons use 544 matched directed store-UPC pairs.}
\label{tab:cross_stability_results}
\begin{tabular}{lccc}
\toprule
\textbf{Metric} & \textbf{OLS} & \textbf{Ridge} & \textbf{MLP} \\
\midrule
Matched bootstrap cross-pairs & 544 & 544 & 544 \\
ICDN narrower CI rate & 85.1\% & 79.4\% & 18.8\% \\
ICDN lower bootstrap std rate & 85.5\% & 79.4\% & 17.3\% \\
Same-sign rate & 49.1\% & 50.4\% & 40.4\% \\
Median cross elasticity, ICDN
& \(0.019\) & \(0.019\) & \(0.019\) \\
Median cross elasticity, benchmark
& \(-0.273\) & \(-0.254\) & \(-0.058\) \\
Positive / negative cross-price signs, ICDN
& \(324/220\) & \(324/220\) & \(324/220\) \\
Positive / negative cross-price signs, benchmark
& \(181/363\) & \(144/400\) & \(0/544\) \\
\midrule
Matched temporal cross-pairs & 544 & 544 & 544 \\
ICDN lower inter-fold std rate & 55.3\% & 41.9\% & 2.6\% \\
Median inter-fold std, ICDN & 0.332 & 0.332 & 0.332 \\
Median inter-fold std, benchmark & 0.284 & 0.204 & 0.027 \\
\bottomrule
\end{tabular}
\end{table}

Under bootstrap resampling, ICDN produces more concentrated cross-price estimates than OLS and Ridge. Its percentile interval is narrower in \(85.1\%\) and \(79.4\%\) of matched pairs, respectively, with similar rates when dispersion is measured by the bootstrap standard deviation. The comparison with the MLP is reversed: ICDN has narrower intervals in only \(18.8\%\) of matched pairs and a lower bootstrap standard deviation in \(17.3\%\).

Unlike the own-price comparison, the lower dispersion of the MLP cannot be explained solely by effect magnitude. Although cross-price effects are small for both neural models, the MLP retains lower relative dispersion after normalization. This may still partly reflect its non-isolated price channel, which can produce a weak or damped cross-price response and therefore less variable Jacobian entries. The results therefore show that ICDN's cross-price estimates are less reproducible under the evaluated protocol, although part of this difference may still arise because the MLP learns a weaker price response and therefore produces less variable cross-price derivatives.

Inter-fold variation provides a weaker and less uniform comparison. ICDN has a lower inter-fold standard deviation than OLS in \(55.3\%\) of matched pairs, but only in \(41.9\%\) relative to Ridge and \(2.6\%\) relative to the MLP. Its median inter-fold standard deviation is \(0.332\), compared with \(0.284\), \(0.204\), and \(0.027\), respectively. 

Finally, cross-price signs show limited agreement across specifications. Same-sign rates are \(49.1\%\), \(50.4\%\), and \(40.4\%\) relative to OLS, Ridge, and the MLP, respectively. ICDN produces 324 positive and 220 negative cross-price estimates, compared with 181/363 for OLS, 144/400 for Ridge, and 0/544 for the MLP. The absence of positive cross-price responses in the MLP is noteworthy in a within-category beer setting, where some substitution effects would be economically plausible, and further highlights the strong model dependence of cross-price estimates.

\paragraph{Resampling-based uncertainty diagnostics.}

Bootstrap resampling and temporal re-estimation capture different sources of variation. Bootstrap intervals measure sensitivity to resampled training data, while inter-fold variation also reflects changes in the estimation period and evaluation context. We therefore compare the nominal \(95\%\) bootstrap intervals with the fold-specific elasticity estimates as a descriptive resampling-consistency diagnostic rather than as a formal coverage or calibration test.

The fold-estimate inclusion rate is defined as the proportion of fold-specific estimates contained within the corresponding bootstrap interval. Bootstrap and inter-fold standard deviations are reported separately to distinguish variation under resampling from variation across temporal estimation windows. Table~\ref{tab:calibration_results} reports these diagnostics for own- and cross-price elasticities.

For own-price elasticities, ICDN exhibits substantially lower bootstrap dispersion than OLS and Ridge, while its inter-fold variation is larger than the dispersion observed across bootstrap re-estimations. Only \(39.7\%\) of its fold-specific estimates fall within the corresponding bootstrap intervals, compared with approximately \(85\%-88\%\) for the benchmarks. This difference should not be interpreted as bootstrap undercoverage. The temporal folds change both the estimation sample and the contexts over which elasticities are evaluated. Since ICDN estimates context-dependent price responses, its greater inter-fold variation may therefore reflect both sensitivity to model re-estimation and genuine changes in the temporal covariate distribution that are not reproduced by bootstrap resampling within a fixed window.

The MLP again requires a separate interpretation. Its bootstrap and inter-fold dispersion are both very small, and its fold-estimate inclusion rate is \(85.1\%\). Its elasticities are also context-dependent, so the same distinction between resampling and temporal variation applies. However, as discussed above, the MLP own-price response is strongly concentrated near zero, which may additionally limit both its bootstrap and temporal variation.

The cross-price results show a similar distinction. ICDN has a median bootstrap standard deviation of \(0.207\) and a larger inter-fold standard deviation of \(0.332\), with a fold-estimate inclusion rate of \(66.3\%\). As with own-price effects, this additional temporal variation may arise from both model re-estimation and changes in the contexts over which the cross-price derivatives are evaluated.

\begin{table}[h!]
\centering
\small
\setlength{\tabcolsep}{10pt}
\renewcommand{\arraystretch}{1.15}
\caption{
Resampling and temporal-variation diagnostics for own- and cross-price elasticities. The fold-estimate inclusion rate reports the proportion of fold-specific estimates contained within the corresponding nominal \(95\%\) bootstrap interval. Because bootstrap and fold estimates are obtained under different temporal windows, this quantity should not be interpreted as confidence-interval coverage of a common estimand.
}
\label{tab:calibration_results}
\begin{tabular}{lcccc}
\toprule
\textbf{Metric}
& \textbf{ICDN}
& \textbf{OLS}
& \textbf{Ridge}
& \textbf{MLP} \\
\midrule

\multicolumn{5}{l}{\textit{Own-price elasticities}} \\

Fold-estimate inclusion rate
& 39.7\%
& 87.9\%
& 85.1\%
& 85.1\% \\

Median bootstrap std
& 0.217
& 0.780
& 0.751
& 0.084 \\

Median inter-fold std
& 0.381
& 0.641
& 0.606
& 0.067 \\

\midrule

\multicolumn{5}{l}{\textit{Cross-price elasticities}} \\

Fold-estimate inclusion rate
& 66.3\%
& 92.8\%
& 93.0\%
& 89.6\% \\

Median bootstrap std
& 0.207
& 0.591
& 0.416
& 0.048 \\

Median inter-fold std
& 0.332
& 0.284
& 0.204
& 0.027 \\

\bottomrule
\end{tabular}
\end{table}

\paragraph{Main takeaway:} ICDN achieves its largest average predictive gain over OLS, with smaller gains over Ridge and the MLP. Its own-price bootstrap estimates are substantially more concentrated than those of OLS and Ridge. The MLP shows lower absolute dispersion but also smaller price responses, while ICDN attains a higher bootstrap SNR in most matched comparisons. Cross-price results are more mixed: ICDN is more concentrated than OLS and Ridge under bootstrap resampling. The MLP shows lower dispersion and temporal variability but produces uniformly negative cross-price estimates, while ICDN recovers both positive and negative responses, more consistent with the coexistence of substitution and complementarity patterns within the beer category. Finally, inter-fold variation also reflects changes in the temporal estimation window and evaluation context, which bootstrap resampling does not capture.

\subsection{Ablation and economic-constraint sensitivity}
\label{sec:ablation}

To assess the contribution of the main ICDN components and the influence of its economic restrictions, we conduct component-ablation and economic-constraint sensitivity analyses on the same four-product application. Each specification is tuned over three expanding temporal folds constructed within the pre-final sample, using one fixed seed and a separate twenty-trial Optuna search. For each component-ablation specification, the best hyperparameter configuration is selected using the combined robust predictive-elasticity criterion of Section~\ref{sec:hyperparameter_optimization}, while economic-constraint variants are selected using robust validation \(R^2\) only, avoiding selection according to the elasticity ranges being tested. The selected configuration is then retrained using the complete initial 80\% sample and assessed on the common final 20\% period. Because this protocol uses fewer temporal folds, one seed, a smaller search budget, and a single final evaluation period, the results are interpreted as diagnostic rather than inferential.


The component ablations remove, in turn, the curvature penalty (\(\lambda_{\mathrm{sm}}=0\)), elasticity-band penalty (\(\lambda_{\mathrm{elast}}=0\)), adaptive cross-price weighting, all cross-price terms, or the spline terms. The economic-constraint analysis considers a free-sign specification, an unconstrained specification removing both the own-price sign transformation and elasticity-band penalty, and a wide-bounds specification using \([-10,2]\) for own-price and \([-3,3]\) for cross-price elasticities.
Table~\ref{tab:ablation_sensitivity} reports holdout predictive performance, the corresponding elasticity score from Section~\ref{sec:hyperparameter_optimization}, and out-of-sample elasticity distribution diagnostics.

\begin{table}[h!]
\centering
\caption{ICDN component ablation and economic-constraint sensitivity on the final holdout after three-fold hyperparameter selection. Changes are relative to the complete ICDN specification. Component ablations are selected using the combined predictive-elasticity criterion, whereas economic-constraint variants are selected using robust validation
\(R^2\) only. ``n.c.'' and ``n.a.'' denote not comparable and not applicable.}
\label{tab:ablation_sensitivity}
\begin{threeparttable}
\resizebox{\textwidth}{!}{%
\begin{tabular}{lrrrrrrrr}
\toprule
Variant
& Holdout $R^2$
& $\Delta R^2$
& $S_{\mathrm{elast}}$
& $\Delta S_{\mathrm{elast}}$
& Own $<0$ (\%)
& Own median
& Own in $[-5,0]$ (\%)
& Cross in $[-1,1]$ (\%) \\
\midrule
\multicolumn{9}{l}{\textit{Panel A: Component ablation}} \\
Full ICDN
& 0.516
& 0.000
& 0.684
& 0.000
& 100.0
& -0.843
& 100.0
& 94.8 \\

No curvature penalty
& 0.514
& -0.003
& 0.620
& -0.065
& 93.9
& -0.816
& 89.3
& 90.3 \\

No elasticity-band penalty
& 0.522
& 0.006
& 0.893
& 0.209
& 100.0
& -1.475
& 100.0
& 90.6 \\

Uniform cross-price weights
& 0.524
& 0.007
& 0.768
& 0.083
& 99.9
& -1.076
& 99.9
& 95.5 \\

No cross-price effects
& 0.532
& 0.016
& \multicolumn{2}{c}{n.c.}
& 100.0
& -1.669
& 99.2
& n.a. \\

No spline terms
& 0.546
& 0.029
& 0.853
& 0.169
& 100.0
& -1.280
& 100.0
& 100.0 \\
\midrule
\multicolumn{9}{l}{\textit{Panel B: Economic-constraint sensitivity}} \\
Free own-price sign
& 0.545
& 0.029
& 0.941
& 0.257
& 100.0
& -1.536
& 100.0
& 99.6 \\

Unconstrained sign and elasticity bands
& 0.549
& 0.032
& 0.853
& 0.168
& 100.0
& -1.349
& 100.0
& 91.9 \\

Wider elasticity bounds
& 0.518
& 0.002
& 0.417
& -0.268
& 94.7
& -0.272
& 89.6
& 79.1 \\
\bottomrule
\end{tabular}%
}
\end{threeparttable}
\end{table}

\paragraph{Component ablations.}
Several simplified variants attain slightly higher predictive \(R^2\) than the full ICDN, although the differences are small. This suggests that its structured price components and derivative-oriented penalties primarily regulate the fitted demand geometry rather than uniformly improving prediction.

The curvature penalty provides the clearest regularization effect: removing it changes \(R^2\) only slightly (\(\Delta R^2=-0.003\)) but reduces the elasticity score by \(0.065\). The own-price in-range share falls from \(100.0\%\) to \(89.3\%\), while the median remains nearly unchanged (\(-0.84\) versus \(-0.82\)). Thus, curvature regularization mainly controls local irregularities and extreme derivatives rather than the center of the own-price distribution.

Removing the elasticity-band penalty slightly improves predictive performance (\(\Delta R^2=0.006\)) and increases the elasticity score from \(0.684\) to \(0.893\), while the cross-price in-range share decreases only moderately from \(94.8\%\) to \(90.6\%\). Thus, the elasticity-band penalty provides some cross-price tail regularization but is not required to obtain broadly plausible elasticity estimates.

Replacing adaptive with uniform cross-price weights also has little predictive effect (\(\Delta R^2=0.007\)) and leaves the elasticity diagnostics broadly comparable, with the cross-price in-range share increasing slightly from \(94.8\%\) to \(95.5\%\). In this four-product application, \(k=3=n-1\), so the comparison concerns adaptive versus uniform weighting on a dense graph rather than sparse neighbor selection.


Removing all cross-price effects improves \(R^2\) by \(0.016\), suggesting some predictive variance cost from cross-product flexibility in this small application. However, this specification cannot estimate directed cross-price elasticities or represent substitution and cannibalization, so its predictive improvement does not imply that cross-price terms are unnecessary for the multiproduct elasticity objective.

Removing the spline terms produces the strongest predictive improvement among the component ablations, increasing \(R^2\) from \(0.516\) to \(0.546\), while the elasticity score rises from \(0.684\) to \(0.853\). The own-price median shifts from \(-0.84\) to \(-1.28\), and both own- and cross-price elasticities remain almost entirely within the reference ranges. Thus, local nonlinear price flexibility is not needed to improve prediction in this four-product scheme, although the spline terms remain necessary for allowing elasticities to vary nonlinearly across pricing regimes.

\paragraph{Economic-constraint sensitivity.}

Removing the hard own-price sign transformation while retaining the elasticity-band penalty leaves essentially all own-price elasticities negative and inside \([-5,0]\), with median \(-1.54\). More importantly, when both the sign transformation and elasticity-band penalty are removed and hyperparameters are selected exclusively by predictive performance, \(100\%\) of own-price elasticities remain negative and within \([-5,0]\), with median \(-1.35\). Thus, the negative own-price pattern is not solely generated by these two restrictions.


The unconstrained specification is not free of economic information: its own-price head remains initialized at \(\beta_{\mathrm{EDA}}=-2\), and the curvature penalty and two-phase protocol remain active. The result should therefore be interpreted narrowly as showing that removal of the sign constraint and elasticity-band penalty does not eliminate the negative out-of-sample own-price pattern.

Widening the elasticity bands has a more pronounced effect. The own-price median shifts from \(-0.84\) to \(-0.27\), while the shares inside the original own- and cross-price reference intervals fall to \(89.6\%\) and \(79.1\%\), respectively. The reference bands therefore affect both the location and tails of the estimated elasticity distribution, although \(94.7\%\) of own-price elasticities remain negative.

\subsection{Targeted architectural alternatives}
\label{sec:architectural_alternatives}

The preceding ablation analysis evaluates the role of individual ICDN components by removing them. We complement it with a controlled substitution experiment testing two architectural choices: the truncated-power spline basis and scaled dot-product attention. The $2\times2$ design in Table~\ref{tab:architecture_variants} combines truncated-power (TP) or natural cubic (NC) price bases with scaled dot-product (DOT) or additive (ADD) attention. Apart from these substitutions, the demand specification in Eq.~\ref{eq:icdn-sparse-g}, sparse top-$k$
mechanism, training objective, data, and evaluation protocol remain unchanged.

\begin{table}[h!]
\centering
\small
\setlength{\tabcolsep}{12pt}
\renewcommand{\arraystretch}{1.15}
\caption{Targeted architectural alternatives. TP-DOT denotes the reference ICDN specification; the remaining variants replace the price basis, attention scorer, or both.}
\label{tab:architecture_variants}
\begin{tabular}{lll}
\toprule
\textbf{Specification}
& \textbf{Price basis}
& \textbf{Attention scorer} \\
\midrule
TP-DOT & Truncated-power cubic & Scaled dot-product \\
NC-DOT & Natural cubic & Scaled dot-product \\
TP-ADD & Truncated-power cubic & Additive \\
NC-ADD & Natural cubic & Additive \\
\bottomrule
\end{tabular}
\end{table}

\paragraph{Natural cubic price basis.}
The reference truncated-power basis and its analytic derivatives are defined in Eq.~\ref{eq:splines}. The alternative replaces $B_i$ in Eq.~\ref{eq:icdn-sparse-g} by a natural cubic basis (\cite{Huang_2021}) of the same nonlinear dimension $K$. For product $i$, let $\xi_{i1}<\cdots<\xi_{i,K+2}$ denote the normalized knots, with boundary knots given by the minimum and maximum training log-prices and $K$ interior knots obtained using the same train-only quantile rule as the reference basis. Define
\[
D_{ij}(z)
=
\frac{
[z-\xi_{ij}]_+^3-[z-\xi_{i,K+2}]_+^3
}{
\xi_{i,K+2}-\xi_{ij}
},
\qquad
N_{ij}(z)
=
D_{ij}(z)-D_{i,K+1}(z),
\quad j=1,\ldots,K\,,
\]
where $[x]_+ = \max(x,0)=\operatorname{ReLU}(x)$ denotes the positive-part operator. The vector $N_i=(N_{i1},\ldots,N_{iK})^\top$ replaces $B_i$ throughout Eq.~\ref{eq:icdn-sparse-g}. This preserves the $K$-dimensional downstream parameterization and $C^2$ smoothness while imposing zero second derivative beyond the boundary knots. First and second derivatives remain analytic.

\paragraph{Additive attention.}
The second substitution replaces only the scaled dot-product compatibility score of Section~\ref{sec:architecture}. Following the additive formulation of \cite{Additive_attention}, we use
\[
s^{\mathrm{ADD}}_{ij}
=
v_a^\top
\tanh\!\left(W_Qh_i+W_Kh_j\right)
+\xi_{ij},
\]
where $v_a$ is a learned scoring vector and $\xi_{ij}$ is the same metadata contribution used by the reference selector. Graph construction, category priority, top-$k$ selection, and
observation-specific softmax normalization remain unchanged. NC-ADD combines both substitutions.

All architectures receive the same tuning opportunity. As an initial architecture-only check, the four variants are evaluated under an identical fixed configuration on the same outer fold using paired training seeds. They then undergo an equal-budget search over eight common hyperparameter configurations and three expanding inner temporal folds, with selection based exclusively on $\overline{R^2}-0.25\,\mathrm{sd}(R^2)$. Elasticity diagnostics do not enter model selection. The best configuration selected separately for each architecture is then evaluated over five expanding outer folds and five paired training seeds. The confirmatory comparison therefore evaluates each architecture under equal tuning opportunity, while the fixed-configuration stage isolates the substitution itself. Table~\ref{tab:architecture_results} reports these results.

\begin{table}[h!]
\centering
\small
\setlength{\tabcolsep}{5.5pt}
\renewcommand{\arraystretch}{1.15}
\caption{Predictive, reproducibility, derivative, and computational diagnostics for the architectural alternatives. $\Delta$MAE is relative to TP-DOT, with paired 95\% week-block bootstrap intervals. Seed sd and fold sd measure MAE variability across training seeds and temporal folds, respectively.}
\label{tab:architecture_results}
\begin{tabular}{lccccc}
\multicolumn{6}{l}{\textit{Panel A: Predictive performance and reproducibility}} \\
\midrule
\textbf{Specification}
& \textbf{MAE}
& \textbf{$\Delta$MAE [95\% CI]}
& \textbf{$R^2$}
& \textbf{Seed sd}
& \textbf{Fold sd} \\
\midrule
TP-DOT
& 0.4588
& $+0.0000$ [$+0.0000$, $+0.0000$]
& 0.5311
& 0.0013
& 0.0233 \\

NC-DOT
& 0.4608
& $+0.0016$ [$-0.0007$, $+0.0040$]
& 0.5265
& 0.0012
& 0.0232 \\

TP-ADD
& 0.4599
& $+0.0012$ [$-0.0000$, $+0.0026$]
& 0.5308
& 0.0010
& 0.0221 \\

NC-ADD
& 0.4564
& $-0.0022$ [$-0.0049$, $+0.0011$]
& 0.5311
& 0.0017
& 0.0242 \\

\addlinespace[0.5em]
\multicolumn{6}{l}{\textit{Panel B: Demand-surface, elasticity, and computational diagnostics}} \\
\midrule
\textbf{Specification}
& \textbf{$P_{99}(|E_{\mathrm{own}}|)$}
& \textbf{$P_{99}(|E_{\mathrm{cross}}|)$}
& \textbf{Curvature RMS}
& \textbf{Time (ms)}
& \textbf{Memory (MB)} \\
\midrule
TP-DOT
& 8.04
& 5.06
& 14.31
& 1.3
& 45.79 \\

NC-DOT
& 8.24
& 2.45
& 9.86
& 1.3
& 45.79 \\

TP-ADD
& 7.50
& 4.50
& 14.19
& 1.3
& 45.79 \\

NC-ADD
& 9.59
& 3.16
& 10.72
& 1.4
& 46.60 \\
\bottomrule
\end{tabular}
\end{table}

Predictive performance is largely insensitive to the two substitutions. NC-DOT and TP-ADD yield slightly higher MAE than the reference, whereas NC-ADD attains the lowest estimate ($\mathrm{MAE}=0.4564$, $\Delta\mathrm{MAE}=-0.0022$). However, all paired 95\% bootstrap intervals include zero, and $R^2$ remains close across specifications. Seed- and fold-level variability is similarly stable, providing no evidence of a clear predictive advantage for any single architecture.

The spline substitution has a more systematic effect on demand-surface geometry. Under dot-product attention, replacing TP by NC reduces curvature RMS ($=\sqrt{N^{-1}\sum_{r=1}^{N}\kappa_r^2}$, where $N$ is the number of evaluation points and $\kappa_r$ denotes the own-price curvature at point $r$) from 14.31 to 9.86, while the 99th percentile of absolute cross-price elasticity falls from 5.06 to 2.45. The same pattern appears under additive attention: curvature RMS falls from 14.19 to 10.72 and the cross-price tail from 4.50 to 3.16. This smoother behavior is accompanied by larger extreme own-price elasticities, especially for NC-ADD. Hence, the natural cubic basis reduces overall curvature and extreme cross-price responses, but does not improve all elasticity diagnostics simultaneously.

The attention substitution produces smaller and less systematic changes. Under the truncated-power basis, additive attention leaves prediction, curvature, and elasticity tails close to TP-DOT. Under the natural cubic basis, NC-ADD achieves the lowest MAE point estimate but somewhat larger curvature and elasticity tails than NC-DOT. Runtime and GPU-memory differences remain negligible across all four architectures.

\subsection{Computational scalability and sparse interaction evaluation}
\label{sec:computational_scalability}

To assess computational scalability independently of predictive performance, we conduct a controlled architecture stress test over product-set sizes \(n \in \{5,25,50,100,200\}\) and different numbers of retained neighbors \(k\), using a fixed reference ICDN configuration held constant across all \(n\) and \(k\) settings. Random weights are used because the experiment targets tensor-shape-dependent runtime and memory rather than predictive accuracy.

Runtimes are averaged over 15 benchmark runs, with each timing block evaluated using 15 CUDA-synchronized repetitions after three warm-up iterations. Memory is measured as incremental peak allocated GPU memory. All experiments use a batch size of 256 on a single NVIDIA GeForce RTX 5070 Ti GPU with CUDA~13.


The fixed reference ICDN configuration used in this stress test contains 61,721 trainable parameters, independent of \(n\) and \(k\) because its main modules are shared across products and directed edges. Increasing \(n\) or \(k\) therefore changes the intermediate activations and evaluated interactions, but not the number of learned weights. Table~\ref{tab:computational_scalability} reports scaling in \(n\) at fixed \(k=4\) and scaling in \(k\) at fixed \(n=200\).


\begin{table*}[h!]
\centering
\caption{Computational scalability of ICDN as the number of products \(n\) and retained neighbors \(k\) increase. Runtime entries report mean \(\pm\) standard deviation over 15 benchmark runs, and memory corresponds to incremental peak GPU allocation. Sparse-output memory contains the \(n\) own-price and \(nk\) selected cross-price elasticities. The case \(k=n-1\) represents a fully dense directed interaction graph. Training uses the current implementation, which materializes a dense \(n\times n\) elasticity tensor for the elasticity-band loss.}
\label{tab:computational_scalability}
\begin{threeparttable}
\small
\resizebox{\textwidth}{!}{%
\begin{tabular}{rrrrrrrr}
\toprule
$n$
& $k$
& Graph density
& \multicolumn{1}{c}{Sparse evaluation}
& \multicolumn{1}{c}{Evaluation memory}
& \multicolumn{1}{c}{Training step}
& \multicolumn{1}{c}{Training memory}
& \multicolumn{1}{c}{Sparse output}
\\
&
&
(\%)
& (ms)
& (MB)
& (ms)
& (MB)
& (MB)
\\
\midrule
\multicolumn{8}{l}{\textit{Panel A: Scaling with the number of products at fixed $k=4$}} \\
\addlinespace[0.3em]
5   & 4 & 100.00 & $0.880 \pm 0.006$ & 6.22   & $4.848 \pm 0.309$ & 8.70   & 0.026 \\
25  & 4 & 16.67  & $1.422 \pm 0.012$ & 31.08  & $6.395 \pm 0.341$ & 42.32  & 0.128 \\
50  & 4 & 8.16   & $2.242 \pm 0.009$ & 64.25  & $8.572 \pm 0.436$ & 85.89  & 0.256 \\
100 & 4 & 4.04   & $4.094 \pm 0.015$ & 125.36 & $13.503 \pm 0.228$ & 198.96 & 0.512 \\
200 & 4 & 2.01   & $8.178 \pm 0.039$ & 249.32 & $26.908 \pm 0.190$ & 562.27 & 1.024 \\
\addlinespace[0.7em]
\multicolumn{8}{l}{\textit{Panel B: Scaling with the number of retained neighbors at fixed $n=200$}} \\
\addlinespace[0.3em]
200 & 4   & 2.01   & $8.178 \pm 0.039$   & 249.32    & $26.908 \pm 0.190$  & 562.27    & 1.024 \\
200 & 8   & 4.02   & $10.609 \pm 0.063$  & 471.18    & $34.551 \pm 0.217$  & 644.39    & 1.843 \\
200 & 16  & 8.04   & $15.340 \pm 0.035$  & 913.20    & $49.958 \pm 0.154$  & 974.94    & 3.482 \\
200 & 32  & 16.08  & $25.423 \pm 0.047$  & 1,797.87  & $81.184 \pm 0.357$  & 1,823.04  & 6.758 \\
200 & 199 & 100.00 & $129.344 \pm 0.154$ & 11,034.47 & $419.497 \pm 0.472$ & 10,675.09 & 40.960 \\
\bottomrule
\end{tabular}%
}
\end{threeparttable}
\end{table*}

At \(k=4\), end-to-end evaluation time increases from \(1.422\) ms for \(n=25\) to \(8.178\) ms for \(n=200\), while incremental peak GPU memory increases from \(31.08\) to \(249.32\) MB. This approximately proportional growth is consistent with dependence on the number of retained edges, \(nk\), when \(k\) is fixed.
For \(n=200\), \(k=4\) retains only \(2.01\%\) of the possible directed cross-product interactions. Relative to the dense case \(k=199=n-1\), sparse evaluation is approximately \(15.8\) times faster and uses \(44.3\) times less incremental peak GPU memory. The corresponding training-step reductions are approximately \(15.6\)-fold in runtime and \(19.0\)-fold in memory. Sparse elasticity output requires \(1.024\) MB, compared with \(40.96\) MB for the dense elasticity matrix.


These results are consistent with approximately \(O(nk)\) behavior for parameter generation, selected-interaction evaluation, and selected-edge elasticity extraction once the interaction support is fixed. Candidate scoring remains \(O(n^2)\) because all directed pairs are evaluated before top-\(k\) selection. For \(n=200\), constructing the frozen graph from one training-split pass requires approximately \(0.43\) seconds and is essentially unaffected by \(k\).

The current training implementation still materializes a dense \(n\times n\) elasticity tensor for the elasticity-band loss and therefore retains an additional quadratic memory component. This dense representation is not required for frozen-graph prediction or selected-edge elasticity extraction, where elasticities can be returned directly in sparse form.
Appendix~\ref{app:computational_scalability} reports the operation-level runtime and trainable parameter-count decompositions.


\subsection{Cross-dataset robustness under a matched nested evaluation protocol}
\label{sec:cross_dataset}

To assess whether the empirical behavior observed in the original Dominick application extends to other retail environments, we repeat the benchmark evaluation on Walmart M5 (\texttt{FOODS\_3}, \cite{makridakis2022m5}), Dunnhumby (\texttt{COLD\_CEREAL}, \cite{dunnhumby2014breakfast}), and Dominick (\texttt{Beer}). The product universe is fixed before model estimation using only the initial 50\% of each panel and pre-specified density and price-variation criteria. In Walmart, eligible SKUs are ranked by positive-sales coverage, store presence, price variation, and sales support, as store-week overlap is already high. In Dunnhumby and Dominick, where overlap is more heterogeneous, the candidate set is additionally refined using pairwise Jaccard store-week co-occurrence together with individual SKU coverage.

Relative to the four-UPC dense specification used in the main Dominick experiment, the product universe is expanded to 20 SKUs in Walmart and Dominick and to 10 SKUs in Dunnhumby. This larger setting also activates ICDN's sparse interaction mechanism. The number of retained neighbors is selected from $k\in\{2,3,4\}$; hence, even at $k=4$, each focal SKU retains at most 44.4\% of the possible off-diagonal neighbors in the 10-product panels and 21.1\% in the 20-product panel.

All models use the same three-fold expanding outer temporal splits and an independent final holdout. For MLP and ICDN, hyperparameters are selected exclusively from three expanding inner folds contained within the corresponding outer-training period. Both neural models receive 100 Optuna trials for each outer split and for the final holdout, and candidate configurations are ranked by mean inner-validation MAE. This auxiliary cross-dataset protocol deliberately uses prediction-only model selection so that ICDN and the MLP are compared under a common architecture-neutral criterion, independently of ICDN-specific elasticity regularization. MLP and ICDN additionally use the same train-fitted contextual feature pipeline. Finally, product metadata are harmonized according to dataset availability: Walmart provides category information, while Dunnhumby and Dominick additionally provide brand, style, and product-size information when available. Table~\ref{tab:cross_dataset_features} summarizes the information available to each estimator.

\begin{table}[h!]
\centering
\small
\setlength{\tabcolsep}{10pt}
\renewcommand{\arraystretch}{1.15}
\caption{Information set used in the cross-dataset evaluation.}
\label{tab:cross_dataset_features}
\begin{tabular}{lcccc}
\toprule
\textbf{Information / feature family}
& \textbf{OLS}
& \textbf{Ridge}
& \textbf{MLP}
& \textbf{ICDN} \\
\midrule
Own log-price, $\log p_i$
& \checkmark & \checkmark & \checkmark & \checkmark \\

Cross-product log-price(s)
& \checkmark & \checkmark & \checkmark & \checkmark \\

Promotion indicator
& \checkmark & \checkmark & \checkmark & \checkmark \\

Annual seasonality, $\sin_{52},\cos_{52}$
& \checkmark & \checkmark & \checkmark & \checkmark \\

Additional calendar / lifecycle features
& -- & -- & \checkmark & \checkmark \\

Lagged and rolling demand features
& -- & -- & \checkmark & \checkmark \\

Lag / rolling missingness indicators
& -- & -- & \checkmark & \checkmark \\

Product metadata
& -- & -- & \checkmark & \checkmark \\

Explicit store conditioning
& -- & -- & \checkmark & \checkmark \\
\bottomrule
\end{tabular}
\end{table}

Table~\ref{tab:cross_dataset_results} reports MAE on store-product-week observations common to all four estimators. Outer results pool the three validation periods. The ranking varies substantially across datasets: ICDN performs best on Walmart, Ridge on Dunnhumby, and the MLP on Dominick. This broader comparison therefore does not support universal predictive superiority of any single model.

\begin{table}[h!]
\centering
\small
\setlength{\tabcolsep}{11pt}
\renewcommand{\arraystretch}{1.15}
\caption{Cross-dataset predictive MAE on store--product--week cells common to all four estimators. Outer results pool the three temporal validation folds. Lower values indicate better predictive performance.}
\label{tab:cross_dataset_results}
\begin{tabular}{llcccc}
\toprule
\textbf{Dataset}
& \textbf{Evaluation}
& \textbf{OLS}
& \textbf{Ridge}
& \textbf{MLP}
& \textbf{ICDN} \\
\midrule
Walmart M5
& Outer
& 0.5418
& 0.5278
& 0.5290
& \textbf{0.4842} \\

& Holdout
& 0.5542
& 0.5525
& 0.5740
& \textbf{0.5187} \\

\addlinespace[0.3em]

Dunnhumby
& Outer
& 0.4029
& \textbf{0.3955}
& 0.4550
& 0.5033 \\

& Holdout
& 0.4497
& \textbf{0.4417}
& 0.5674
& 0.5329 \\

\addlinespace[0.3em]

Dominick
& Outer
& 0.4575
& 0.4540
& \textbf{0.4483}
& 0.6166 \\

& Holdout
& 0.4870
& 0.4786
& \textbf{0.4353}
& 0.5023 \\
\bottomrule
\end{tabular}
\end{table}

Because ICDN-MLP is the primary architectural comparison, we additionally evaluate the two neural models on their complete set of matched prediction cells. Define
\[
\Delta\mathrm{MAE}
=
\mathrm{MAE}_{\mathrm{ICDN}}
-
\mathrm{MAE}_{\mathrm{MLP}},
\]
so that negative values favor ICDN. In the pooled outer evaluation, Walmart favors ICDN ($\Delta\mathrm{MAE}=-0.0453$), with the 95\% week-block bootstrap interval lying below zero. The Dunnhumby point estimate favors the MLP ($\Delta\mathrm{MAE}=0.0487$), although its interval spans zero, whereas Dominick clearly favors the MLP ($\Delta\mathrm{MAE}=0.1683$). The independent holdout preserves the Walmart and Dominick rankings, with $\Delta\mathrm{MAE}=-0.0560$ and $0.0670$, respectively, while Dunnhumby reverses in favor of ICDN ($\Delta\mathrm{MAE}=-0.0365$). Figure~\ref{fig:cross_dataset_mae} summarizes these results.

\begin{figure}[h!]
    \centering
    \includegraphics[width=0.70\textwidth]{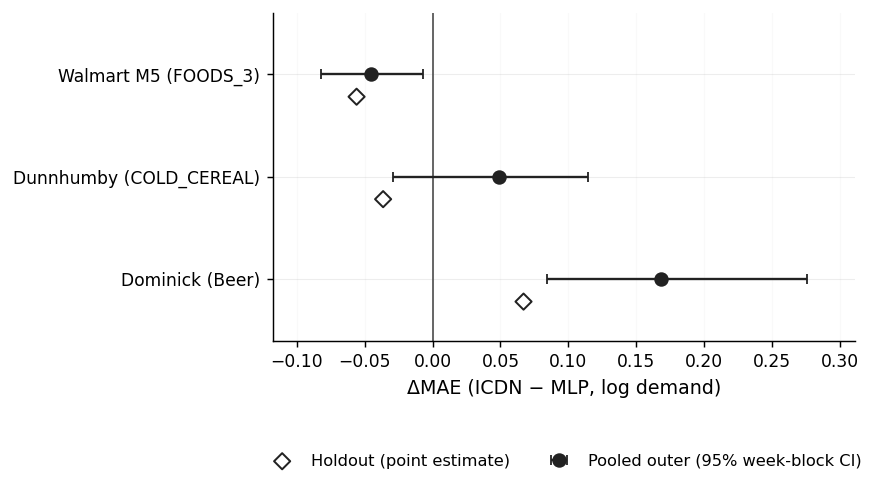}
    \caption{Paired predictive comparison between ICDN and the MLP
    across datasets. Filled circles report pooled outer-validation
    $\Delta\mathrm{MAE}$ with 95\% week-block bootstrap intervals;
    open diamonds report independent holdout point estimates.
    Negative values favor ICDN and positive values favor the MLP.
    The intervals capture temporal resampling variability of the
    fitted-model prediction difference and do not include
    model-selection or retraining uncertainty.}
    \label{fig:cross_dataset_mae}
\end{figure}

The elasticity analysis leads to a similarly heterogeneous conclusion. Using matched bootstrap replicates, ICDN exhibits a lower own-price elasticity standard deviation than the MLP in 81.0\% of matched Walmart series, compared with 40.7\% in Dunnhumby and 34.8\% in Dominick. The corresponding cross-price rates are 30.1\%, 28.7\%, and 40.4\%, respectively. Thus, the own-price concentration advantage is pronounced in Walmart but does not generalize uniformly, while cross-price derivatives remain particularly sensitive to model specification.

Figure~\ref{fig:cross_dataset_elasticities} complements these stability comparisons with the pooled outer-validation elasticity distributions. OLS and Ridge generally produce substantially wider distributions, whereas the neural estimates are more concentrated. ICDN own-price elasticities remain predominantly negative, although their location varies markedly across datasets. Its cross-price estimates are shifted toward positive values, while the MLP remains more tightly centered around zero.

\begin{figure}[h!]
    \centering
    \includegraphics[width=0.8\textwidth]{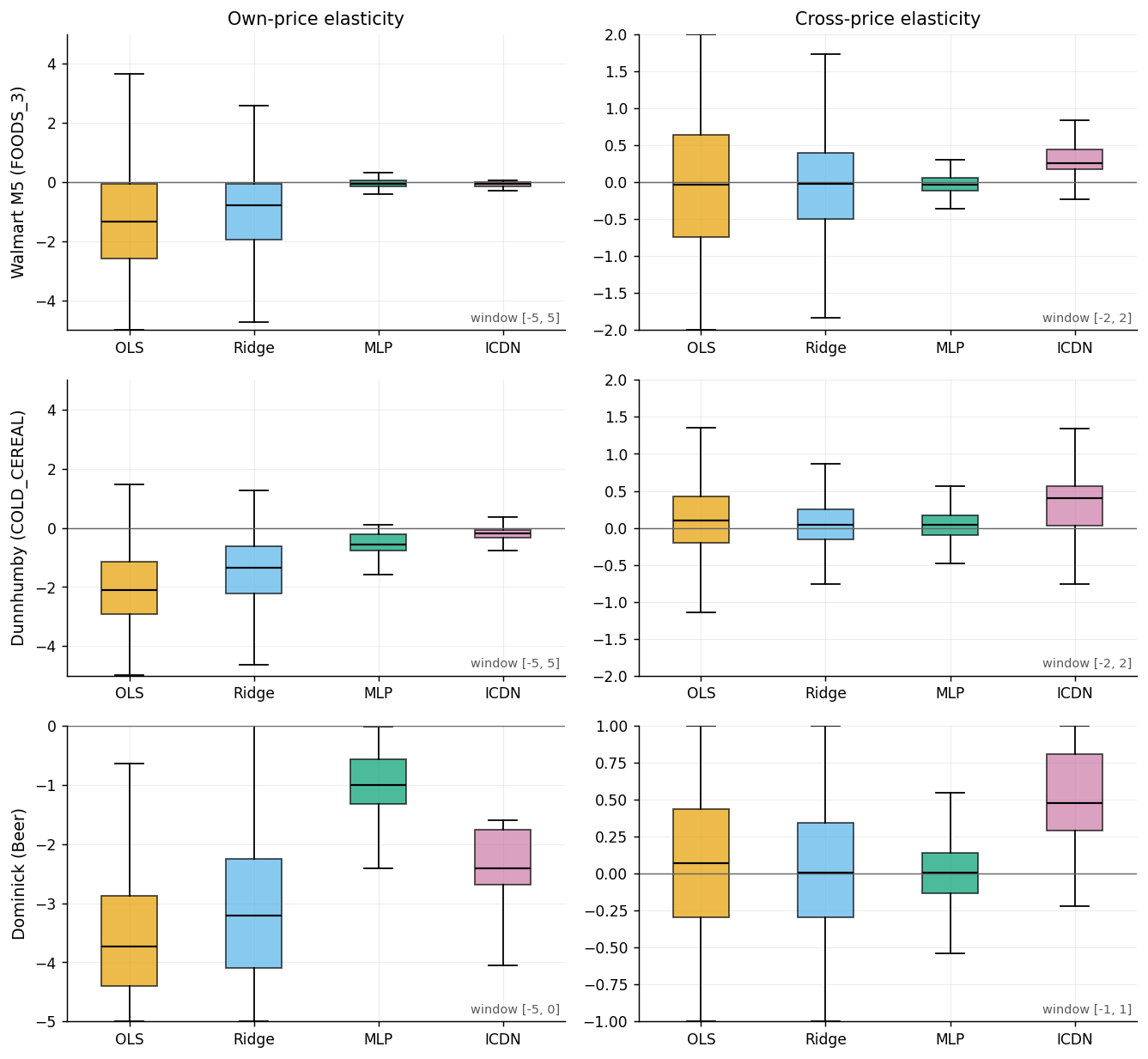}
    \caption{Pooled outer-validation own- and cross-price elasticity distributions across datasets. Boxplots combine the three temporal folds; values outside the displayed windows are clipped for visualization only.}
    \label{fig:cross_dataset_elasticities}
\end{figure}

Overall, ICDN performs particularly well on Walmart, but its predictive and elasticity-stability advantages are not systematic across retail environments. The results therefore position ICDN as a compelling structured alternative to generic neural and local log-log demand models, rather than as a universally superior estimator.

\section{Conclusions}
\label{sec:conclusion}

This paper addresses multiproduct elasticity estimation from a demand-first perspective: log-demand is modeled as a smooth, context-dependent function of log-prices and covariates, and elasticities are obtained as derivatives of the same fitted surface. Consequently, any sufficiently smooth demand-first specification yields row-wise integrable elasticities by construction.

We instantiate this principle through the Integrable Context-Dependent Demand Network (ICDN), which combines contextual conditioning, product-specific spline price bases, analytic derivatives, and sparse directed cross-price interactions. The resulting elasticities are nonlinear, context-dependent, directional, and analytically tractable. ICDN does not impose Hicksian or Slutsky symmetry, so the effect of product \(j\)'s price on product \(i\)'s demand need not equal the reverse effect.

On the four-UPC Dominick application, ICDN achieves its largest average predictive gain relative to OLS, with smaller gains over Ridge and the feature-matched MLP. The comparison with the MLP is particularly informative because both neural models use the same contextual information and nested temporal tuning framework, while ICDN retains an explicit analytically differentiable price-response structure. However, the broader cross-dataset evaluation shows that predictive rankings are dataset-dependent: ICDN performs best on Walmart, Ridge on Dunnhumby, and the MLP on Dominick. The evidence therefore positions ICDN as a compelling candidate for a structured demand model, rather than as a universally superior predictor.


The own-price results provide a more nuanced comparison. ICDN produces substantially more concentrated bootstrap distributions and lower temporal variation than OLS and Ridge, whereas the MLP exhibits lower absolute dispersion. However, the MLP also recovers considerably smaller price responses. After normalizing dispersion by effect magnitude, ICDN attains a higher bootstrap signal-to-noise ratio than the MLP in most matched series. The ablation analysis indicates that curvature regularization contributes to controlling local irregularities, while several simplified variants can slightly improve holdout prediction. The elasticity restrictions primarily shape the fitted demand geometry rather than uniformly improving predictive fit: widening the elasticity bands affects both the tails and the location of the resulting elasticity distribution.

The controlled architectural substitutions reinforce this trade-off. Natural cubic splines substantially reduce demand-surface curvature and extreme cross-price derivatives under both attention mechanisms, although they can increase extreme own-price responses. Additive attention, by contrast, provides no consistent predictive or derivative advantage over scaled dot-product attention. None of the four specifications uniformly dominates the others.

The negative own-price pattern is not solely mechanical. When both the hard own-price sign transformation and the elasticity-band penalty are removed, and hyperparameters are selected only by predictive performance, the out-of-sample own-price elasticities remain negative and inside \([-5,0]\). However, the \(\beta_{\mathrm{EDA}}=-2\) initialization, curvature penalty, and two-phase training protocol remain active, so this result should be interpreted narrowly.

Cross-price elasticities remain substantially more specification-sensitive. ICDN produces more concentrated bootstrap distributions than OLS and Ridge, but not than the MLP, and its temporal-stability advantage is weak and non-uniform. Sign agreement across model classes is also low. The positive ICDN responses are compatible with within-category substitution and cannibalization, but should be interpreted as model-implied patterns rather than evidence that the true cross-product structure has been recovered.

The uncertainty diagnostics further qualify these results. For ICDN, bootstrap dispersion is smaller than the variation observed across temporal re-estimation windows for both own- and cross-price elasticities. The corresponding fold-estimate inclusion rates are 39.7\% and 66.3\%, respectively. Because temporal folds change both the estimation sample and the covariate contexts at which derivatives are evaluated, these quantities are not formal coverage rates. Nevertheless, the bootstrap intervals should be interpreted as resampling-sensitivity diagnostics rather than as fully calibrated uncertainty intervals.


The controlled scalability experiment shows substantial computational gains from selected-edge evaluation as the product universe expands. Once the interaction support is fixed, the main evaluation operations scale approximately with \(nk\), although candidate scoring remains quadratic in \(n\) and the current training implementation retains a dense elasticity tensor for regularization. The four-UPC main application is effectively dense, whereas the larger cross-dataset experiments activate the sparse top-\(k\) mechanism.

A central limitation is price endogeneity. Retail prices may respond to unobserved demand shocks and be jointly determined with promotions, inventory, assortment, and competing-product prices. The reported elasticities should therefore be interpreted as local, context-conditioned derivatives of the fitted observational demand surface, not as causal responses to externally imposed price changes. Integrability and derivative regularization do not remove endogeneity or omitted-variable bias, and the observed controls and sample filters cannot rule out residual confounding. Randomized or otherwise exogenous price variation would provide a more direct assessment of causal
price responses.

For applied pricing, ICDN provides an explicit and analytically tractable mapping from demand predictions to local price responses, which can support simulation, pricing diagnostics, and scenario analysis. The own-price derivatives show favorable relative stability in the main application, but this advantage is dataset-dependent. Cross-price effects require greater caution because their magnitude, sign, and stability remain strongly specification-sensitive; operational use should therefore incorporate temporal sensitivity analysis or conservative uncertainty estimates.


Future work should further refine the directed attention mechanism to improve the identification and stability of economically meaningful cross-price interactions, particularly in settings with sparse or heterogeneous substitution patterns. Moreover, the directed cross-price structure could further be studied as a product-interaction network linking inferred substitution patterns to product attributes and observed market behavior. It should also develop uncertainty procedures that better capture changes in demand relationships over time and temporal dependence in the data. Finally, ICDN should be evaluated in revenue, margin, and promotion-optimization tasks, ideally using randomized or otherwise exogenous price variation to distinguish predictive derivatives from causal intervention effects.

\section*{Code and data availability}
The code and processed data used to reproduce the preprocessing, training, and empirical analyses are publicly available at \texttt{github.com/carlosherediapimienta/nn-elasticity} and \texttt{github.com/carlosherediapimienta/nn-elasticity-additional-work}. ICDN is also distributed as a standalone Python package and can be installed via \texttt{pip install icdn}.

\section*{Acknowledgments}
CH: To my beloved Eva and my precious Daniela, thank you for being my endless inspiration and the driving force behind everything I do.

DR: I would like to express my gratitude to Carlos Heredia for taking the time to teach me.

Both authors gratefully acknowledge DAMM's Artificial Intelligence Department (IAMM) for the valuable discussions on this topic and for the encouragement and motivation provided throughout this work, in particular, Daniel Losada and Lluis Pallarès.

\section*{Author Contributions}
Both authors contributed equally to this work.

\section*{Funding}
The authors declare that no funding was received for this work.

\appendix

\section{Appendix}

\subsection{Example: constant elasticity matrix} \label{Appendix:Constant_Elasticity}

To illustrate the integrability condition, consider a multiproduct log-log specification with a constant elasticity matrix at a fixed context, with
$E_{ij}(p)\equiv c_{ij}$ for $i\neq j$, $E_{ii}(p)\equiv \varepsilon_i$, and \(c_{ij},\varepsilon_i\in\mathbb{R}\). For SKU \(i\), the associated \(1\)-form is
\[
\omega_i(p)
=
\sum_{j=1}^n E_{ij}(p)\,d(\log p_j)
=
\varepsilon_i\,d(\log p_i)
+
\sum_{j\neq i} c_{ij}\,d(\log p_j).
\]
Because the coefficients are constant in log-price coordinates, 
\[
\omega_i
=
d\left(
\varepsilon_i\log p_i
+
\sum_{j\neq i} c_{ij}\log p_j
\right).
\]
Hence, \(d\omega_i=0\), so the elasticity field is integrable and path-independent. Given a baseline price vector \(p_0\in\mathbb{R}_{++}^n\) and a positive baseline demand \(v_i(p_0)>0\), a corresponding log-demand function is
\[
\log v_i(p)
=
\log v_i(p_0)
+
\varepsilon_i\bigl(\log p_i-\log p_{i0}\bigr)
+
\sum_{j\neq i}
c_{ij}\bigl(\log p_j-\log p_{j0}\bigr).
\]
Equivalently, demand admits the closed-form representation
\[
v_i(p)
=
v_i(p_0)
\left(\frac{p_i}{p_{i0}}\right)^{\varepsilon_i}
\prod_{j\neq i}
\left(\frac{p_j}{p_{j0}}\right)^{c_{ij}}.
\]
For typical retail goods, \(\varepsilon_i<0\), so lowering the own price increases demand. Writing prices relative to the baseline as $p_j=p_{j0}(1-d_j)$, with $d_j\in[0,1)$ denoting the proportional discount, then
\[
v_i(p)
=
v_i(p_0)\,
(1-d_i)^{\varepsilon_i}
\prod_{j\neq i}(1-d_j)^{c_{ij}}.
\]
Under this convention,  \(c_{ij}>0\) indicates substitution: a higher price for product $j$ raises demand for product $i$, while a discount on $j$ cannibalizes $i$. Conversely, \(c_{ij}<0\) indicates complementarity.

This benchmark is analytically transparent and path-independent. Given baseline sales, it provides a closed-form projection of demand under any price configuration and an interpretable benchmark for multiproduct scenario analysis.

This log-linear specification is closely related to classical retail price-response models such as SCAN*PRO and its extensions (\cite{wittink1988scan,vanHeerdeLeeflangWittink2002}). However, it remains restrictive because it imposes globally constant own- and cross-price effects and cannot capture nonlinear or context-dependent responses.

\section{Computational scalability details}
\label{app:computational_scalability}

This appendix reports the operation-level decomposition of the computational stress test and the trainable parameter count of the fixed reference ICDN configuration used in Section~\ref{sec:computational_scalability}.

\paragraph{Table~\ref{tab:computational_breakdown}} decomposes the computational cost into the following operations: (a) \emph{candidate scoring}, which evaluates the dense directed score matrix and therefore retains an \(O(n^2)\) component; (b) \emph{top-\(k\) selection}, which retains the highest-ranked non-self candidates for each focal product; (c) \emph{parameter/own terms}, including context-dependent parameter generation, attention weighting, and own-price demand and elasticity terms; (d) \emph{cross-interaction evaluation} over the \(nk\) selected edges; (e) \emph{selected-edge elasticity computation} over those same edges; (f) \emph{dense elasticity-matrix assembly}, which scatters the computed own- and selected cross-price elasticities into an \(n\times n\) tensor, setting non-selected entries to zero; and (g) \emph{graph freezing} averages the learned candidate scores over one complete pass of the training split and applies top-\(k\) selection once. Its estimated cost is
\[
T_{\mathrm{freeze}}
=
70\left(
T_{\mathrm{tokens}}
+
T_{\mathrm{encoder}}
+
T_{\mathrm{candidate}}
\right)
+
T_{\mathrm{top}\text{-}k},
\tag{27}
\]
where the first three terms are per-batch token-construction, encoding, and candidate-scoring times, and \(T_{\mathrm{top}\text{-}k}\) is the final neighbor-selection time. The factor \(70=\lceil 17{,}795/256\rceil\) is the number of batches in one training-split pass. Because candidate scoring precedes neighbor selection, graph-freezing cost depends primarily on \(n\) and is essentially unaffected by \(k\). The reported freezing times are constructed from mean component times and therefore have no separate standard deviation.

\begin{table*}[h!]
\centering
\caption{Operation-level computational cost decomposition for \(n=200\). Runtime entries report mean \(\pm\) standard deviation over 15 benchmark runs. Graph freezing is the estimated cost of one complete training-split pass.}
\label{tab:computational_breakdown}
\begin{threeparttable}
\scriptsize
\resizebox{\textwidth}{!}{%
\begin{tabular}{rrrrrrrr}
\toprule
$k$
& Candidate scoring
& Top-$k$
& Parameter/own terms
& Cross interactions
& Selected elasticities
& Dense $E$
& Graph freezing
\\
& (ms/batch)
& (ms)
& (ms)
& (ms)
& (ms)
& (ms)
& (ms)
\\
\midrule
4
& $0.680 \pm 0.004$
& $0.076 \pm 0.001$
& $1.131 \pm 0.002$
& $0.929 \pm 0.003$
& $0.480 \pm 0.002$
& $0.062 \pm 0.001$
& 430.44
\\
8
& $0.680 \pm 0.004$
& $0.076 \pm 0.001$
& $2.143 \pm 0.004$
& $1.840 \pm 0.003$
& $0.938 \pm 0.003$
& $0.066 \pm 0.0004$
& 430.95
\\
16
& $0.680 \pm 0.003$
& $0.076 \pm 0.0005$
& $4.136 \pm 0.007$
& $3.643 \pm 0.006$
& $1.843 \pm 0.003$
& $0.080 \pm 0.002$
& 430.97
\\
32
& $0.680 \pm 0.005$
& $0.075 \pm 0.001$
& $7.981 \pm 0.012$
& $7.647 \pm 0.007$
& $3.863 \pm 0.005$
& $0.117 \pm 0.002$
& 431.93
\\
199
& $0.680 \pm 0.004$
& $0.074 \pm 0.001$
& $48.515 \pm 0.071$
& $48.244 \pm 0.044$
& $24.542 \pm 0.026$
& $0.202 \pm 0.001$
& 429.27
\\
\bottomrule
\end{tabular}%
}
\end{threeparttable}
\end{table*}

Candidate scoring and top-\(k\) selection remain nearly constant as \(k\) increases, whereas parameter generation, cross-interaction evaluation, and selected-edge elasticity computation grow approximately with the number of retained neighbors. Graph-freezing time remains close to \(0.44\) seconds. Thus, the main computational benefit of sparsity arises after the interaction support has been selected.

\paragraph{Table \ref{tab:parameter_count}} reports the 61,721 trainable parameters of the fixed reference ICDN configuration. Most are concentrated in the shared product encoder, and the total count is invariant to \(n\) and \(k\) because the same modules are reused across products and selected edges.
\begin{table}[t]
\centering
\caption{Trainable parameter count of the fixed reference ICDN configuration. The count is invariant to \(n\) and \(k\). Spline knots, normalization constants, pair indices, and the frozen interaction graph are non-trainable and are therefore excluded.}
\label{tab:parameter_count}
\begin{threeparttable}
\small
\begin{tabular}{lr}
\toprule
Component & Trainable parameters \\
\midrule
Product token builder    & 1,156 \\
Shared product encoder   & 56,512 \\
Demand parameter head    & 2,002 \\
Sparse neighbor selector & 2,051 \\
\midrule
\textbf{Total ICDN}      & \textbf{61,721} \\
\bottomrule
\end{tabular}
\end{threeparttable}
\end{table}

\section{Feature tables}

This section summarizes the variables used to train the MLP and ICDN models in the main Dominick application and provides their corresponding definitions and modeling rationale. The tables organize the feature set by information family: product identifiers and descriptors, demand, price and promotion variables, temporal and lifecycle controls, demand-history and missingness features, and dynamic and static competitive-neighborhood information. Together, they document the complete information set available to both neural models and clarify the role of each variable in the empirical specification.

\small
\setlength{\tabcolsep}{4pt}
\renewcommand{\arraystretch}{1.15}


\begin{table}[h!]
\centering
\small
\caption{Final dataset schema: identifiers and product descriptors.}
\label{tab:final-dl-schema-identifiers}
\begin{tabularx}{\textwidth}{@{}p{0.43\textwidth}X@{}}
\toprule
\textbf{Variable(s)} & \textbf{Definition and motivation} \\
\midrule

\path|store_code|
& Store identifier. Captures persistent store-specific heterogeneity in baseline demand, assortment, and local consumer behavior. \\

\path|upc_code|
& Product identifier. Captures UPC-level heterogeneity in baseline demand, product appeal, and price sensitivity. \\

\path|brand_family_norm|
& Normalized brand-family label. Provides brand membership information and supports same-brand neighborhood features. \\

\path|style_segment_norm|
& Normalized style-segment label. Adds product-descriptor information beyond category, allowing the model to distinguish heterogeneous product styles. \\

\path|category_code|
& Product-category identifier. Defines higher-level product groupings and category-based competitive neighborhoods. \\

\path|liters_per_upc|
& Liters contained in one sellable UPC unit. Controls for package-size differences and supports price and quantity normalization. \\

\bottomrule
\end{tabularx}
\end{table}


\begin{table}[p]
\centering
\small
\caption{Final dataset schema: target, price, and promotion variables.}
\label{tab:final-dl-schema-price}
\begin{tabularx}{\textwidth}{@{}p{0.43\textwidth}X@{}}
\toprule
\textbf{Variable(s)} & \textbf{Definition and motivation} \\
\midrule

\path|log_liters_sold|
& Log demand in liters. This is the supervised learning target \(y\), retained for training and evaluation but not used as an input when predicting demand. \\

\path|log_price_per_liter|
& Log price per liter. This is the main economic price input $u$, from which own-price and cross-price elasticities are derived. It enters only through the explicit price branch and is not included among the contextual covariates. \\

\path|on_promo|
& Indicator that the UPC is on promotion in the corresponding store-week. Controls promotional status, which is strongly related to price and relevant for isolating price effects. \\

\path|promo_intensity_store_week|
& Share of UPCs on promotion in the same store-week. Summarizes the store-wide promotional environment beyond the focal UPC's own promotion status. \\

\bottomrule
\end{tabularx}
\end{table}


\begin{table}[p]
\centering
\small
\caption{Final dataset schema: temporal, seasonal, and lifecycle features.}
\label{tab:final-dl-schema-temporal}
\begin{tabularx}{\textwidth}{@{}p{0.43\textwidth}X@{}}
\toprule
\textbf{Variable(s)} & \textbf{Definition and motivation} \\
\midrule

\path|week_id|
& Raw calendar week identifier. Retained for ordering, traceability, and downstream joins. \\

\path|week_rank|
& Sequential gap-free time index constructed after calendar completion. Captures smooth global time trends. \\

\begin{minipage}[t]{\linewidth}
\path|sin_52|\\
\path|cos_52|
\end{minipage}
& Annual Fourier seasonality terms. Capture smooth yearly demand patterns without high-dimensional week fixed effects. \\

\begin{minipage}[t]{\linewidth}
\path|sin_26|\\
\path|cos_26|
\end{minipage}
& Semiannual Fourier seasonality terms. Represent medium-horizon seasonal cycles beyond annual seasonality. \\

\begin{minipage}[t]{\linewidth}
\path|sin_13|\\
\path|cos_13|
\end{minipage}
& Quarterly Fourier seasonality terms. Capture shorter recurrent seasonal patterns in demand. \\

\path|weeks_since_first_seen_upc|
& Weeks since the UPC first appears in the data. Captures global product lifecycle and product maturity. \\

\path|weeks_since_first_seen_store_upc|
& Weeks since the UPC first appears in a given store. Captures local lifecycle effects such as rollout, stabilization, or decline within store. \\

\bottomrule
\end{tabularx}
\end{table}


\begin{table}[p]
\centering
\small
\caption{Final dataset schema: demand-memory features and missingness indicators.}
\label{tab:final-dl-schema-demand-memory}
\begin{tabularx}{\textwidth}{@{}p{0.43\textwidth}X@{}}
\toprule
\textbf{Variable(s)} & \textbf{Definition and motivation} \\
\midrule

\begin{minipage}[t]{\linewidth}
\path|lag_1_log_liters_sold|\\
\path|lag_2_log_liters_sold|\\
\path|lag_4_log_liters_sold|
\end{minipage}
& Lagged log demand for the same store-UPC series at 1, 2, and 4 weeks. Captures short- and medium-run persistence in demand. \\

\begin{minipage}[t]{\linewidth}
\path|rolling_mean_4_log_liters_sold|\\
\path|rolling_mean_13_log_liters_sold|
\end{minipage}
& Four- and thirteen-week rolling averages of log demand. Summarize recent and quarterly local demand levels while smoothing week-level noise. \\

\begin{minipage}[t]{\linewidth}
\path|miss_lag_1|\\
\path|miss_lag_2|\\
\path|miss_lag_4|
\end{minipage}
& Missingness indicators for lagged demand features. Identify cases in which lagged demand information is unavailable. \\

\begin{minipage}[t]{\linewidth}
\path|miss_roll_4|\\
\path|miss_roll_13|
\end{minipage}
& Missingness indicators for rolling demand features. Identify cases in which rolling demand summaries cannot be computed because the required historical window is unavailable. \\

\bottomrule
\end{tabularx}
\end{table}


\begin{table}[p]
\centering
\small
\caption{Final dataset schema: dynamic competitive-neighborhood features.}
\label{tab:final-dl-schema-competitive-dynamic}
\begin{tabularx}{\textwidth}{@{}p{0.43\textwidth}X@{}}
\toprule
\textbf{Variable(s)} & \textbf{Definition and motivation} \\
\midrule

\path|n_neighbors_sw_cat|
& Number of other UPCs observed in the same store-week-category. Measures the size of the active category-level competitive set. \\

\path|neighbor_promo_share_sw_cat|
& Share of category neighbors on promotion in the same store-week. Captures promotional pressure from rival UPCs within the category. \\

\path|n_same_brand_neighbors_sw_cat|
& Number of same-brand neighbors in the same store-week-category. Measures within-category same-brand competitive intensity. \\

\path|same_brand_neighbor_promo_share_sw_cat|
& Share of same-brand neighbors on promotion. Captures promotional pressure among close same-brand substitutes. \\

\path|lag1_neighbor_mean_log_liters_sold|
& Mean lag-1 log demand across valid category neighbors. Summarizes recent demand conditions in the local competitive set. \\

\path|lag1_same_brand_neighbor_mean_log_liters_sold|
& Mean lag-1 log demand across valid same-brand neighbors. Captures recent demand conditions among same-brand neighboring UPCs. \\

\path|roll4_neighbor_mean_log_liters_sold|
& Mean four-week rolling log demand across valid category neighbors. Provides a smoother measure of recent category-level demand conditions. \\

\begin{minipage}[t]{\linewidth}
\path|miss_lag1_neighbor_mean_log_liters_sold|\\
\path|miss_roll4_neighbor_mean_log_liters_sold|\\
\path|miss_lag1_same_brand_neighbor_mean_log_liters_sold|
\end{minipage}
& Missingness indicators for neighbor-based demand aggregates. Identify cases in which neighbor demand information is unavailable. 
\\

\bottomrule
\end{tabularx}
\end{table}


\FloatBarrier

\begin{table}[H]
\centering
\small
\caption{Final dataset schema: static assortment and neighbor-entry features.}
\label{tab:final-dl-schema-competitive-static}
\begin{tabularx}{\textwidth}{@{}p{0.43\textwidth}X@{}}
\toprule
\textbf{Variable(s)} & \textbf{Definition and motivation} \\
\midrule

\path|store_category_upc_count_static|
& Distinct UPC count in the same store-category. Controls for product range and fixed competitive scale at the store-category level. \\

\path|same_brand_upc_count_store_cat_static|
& Same-brand UPC count in the same store-category, excluding the focal UPC. Measures within-category same-brand variety and potential same-brand substitution. \\

\begin{minipage}[t]{\linewidth}
\path|n_new_neighbors_13w|\\
\path|share_new_neighbors_13w|
\end{minipage}
& Count and share of neighbors that are new in-store, defined as UPCs with at most 13 weeks since first appearance in that store. Capture changing competitive sets due to product introductions, rollouts, or activation. \\

\bottomrule
\end{tabularx}
\end{table}

\normalsize

\bibliography{iclr_conference}
\bibliographystyle{iclr_conference}

\end{document}